\documentclass{article}

\usepackage[preprint]{neurips_2026}

\usepackage[utf8]{inputenc} % allow utf-8 input
\usepackage[T1]{fontenc}    % use 8-bit T1 fonts
\usepackage{hyperref}       % hyperlinks
\usepackage{url}            % simple URL typesetting
\usepackage{booktabs}       % professional-quality tables
\usepackage{amsfonts}       % blackboard math symbols
\usepackage{nicefrac}       % compact symbols for 1/2, etc.
\usepackage{microtype}      % microtypography
\usepackage{xcolor}         % colors

\usepackage{amsmath}
\usepackage{mathtools}

\usepackage{hyperref}
\usepackage{url}
\usepackage[utf8]{inputenc} % allow utf-8 input
\usepackage[T1]{fontenc}    % use 8-bit T1 fonts
\usepackage{hyperref}       % hyperlinks

\usepackage{caption}
\usepackage[dvipsnames]{xcolor}
\usepackage[super]{nth}
\usepackage{graphicx}
\usepackage{physics}
\usepackage{enumitem}
\usepackage{amssymb}
\usepackage{amsthm}
\usepackage{wrapfig}
\usepackage{algorithm}
\usepackage{algorithmic}

\makeatletter
\let\save@mathaccent\mathaccent
\newcommand*\if@single[3]{%
  \setbox0\hbox{${\mathaccent"0362{#1}}^H$}%
  \setbox2\hbox{${\mathaccent"0362{\kern0pt#1}}^H$}%
  \ifdim\ht0=\ht2 #3\else #2\fi
  }
\newcommand*\rel@kern[1]{\kern#1\dimexpr\macc@kerna}
\newcommand*\widebar[1]{\@ifnextchar^{{\wide@bar{#1}{0}}}{\wide@bar{#1}{1}}}
\newcommand*\wide@bar[2]{\if@single{#1}{\wide@bar@{#1}{#2}{1}}{\wide@bar@{#1}{#2}{2}}}
\newcommand*\wide@bar@[3]{%
  \begingroup
  \def\mathaccent##1##2{%
    \let\mathaccent\save@mathaccent
    \if#32 \let\macc@nucleus\first@char \fi
    \setbox\z@\hbox{$\macc@style{\macc@nucleus}_{}$}%
    \setbox\tw@\hbox{$\macc@style{\macc@nucleus}{}_{}$}%
    \dimen@\wd\tw@
    \advance\dimen@-\wd\z@
    \divide\dimen@ 3
    \@tempdima\wd\tw@
    \advance\@tempdima-\scriptspace
    \divide\@tempdima 10
    \advance\dimen@-\@tempdima
    \ifdim\dimen@>\z@ \dimen@0pt\fi
    \rel@kern{0.6}\kern-\dimen@
    \if#31
      \overline{\rel@kern{-0.6}\kern\dimen@\macc@nucleus\rel@kern{0.4}\kern\dimen@}%
      \advance\dimen@0.4\dimexpr\macc@kerna
      \let\final@kern#2%
      \ifdim\dimen@<\z@ \let\final@kern1\fi
      \if\final@kern1 \kern-\dimen@\fi
    \else
      \overline{\rel@kern{-0.6}\kern\dimen@#1}%
    \fi
  }%
  \macc@depth\@ne
  \let\math@bgroup\@empty \let\math@egroup\macc@set@skewchar
  \mathsurround\z@ \frozen@everymath{\mathgroup\macc@group\relax}%
  \macc@set@skewchar\relax
  \let\mathaccentV\macc@nested@a
  \if#31
    \macc@nested@a\relax111{#1}%
  \else
    \def\gobble@till@marker##1\endmarker{}%
    \futurelet\first@char\gobble@till@marker#1\endmarker
    \ifcat\noexpand\first@char A\else
      \def\first@char{}%
    \fi
    \macc@nested@a\relax111{\first@char}%
  \fi
  \endgroup
}
\makeatother

\newcommand{\shorteq}[1]{\mathrel{\makebox[#1][c]{=}}}

\theoremstyle{plain}
\newtheorem{theorem}{Theorem}[section]
\newtheorem{proposition}[theorem]{Proposition}

\theoremstyle{definition}
\newtheorem{definition}[theorem]{Definition}

\theoremstyle{remark}

\title{Robust Conditional Conformal Prediction via\\ Branched Normalizing Flow}

\author{%
  Rui Xu$^{1}$, Xingyuan Chen$^{1}$, Wenxing Huang$^{2}$, Minxuan Huang$^{2}$, \\ 
  \textbf{Weiyan Chen}$^2$, \textbf{Sihong Xie}$^1$, \textbf{Hui Xiong}$^1$ \\
  $^{1}$ Information Hub, The Hong Kong University of Science and Technology (Guangzhou)\\
  $^{2}$ The Second Affiliated Hospital of Guangzhou Medical University
}
\begin{document}

\maketitle

\vspace{-10pt}
\begin{abstract}
\vspace{-5pt}
Conformal prediction (CP) constructs prediction sets with marginal coverage guarantees under the assumption that the calibration and test distributions are identical. However, under distribution shift, existing approaches primarily align marginal conformal score distributions, which is sufficient to preserve marginal coverage but does not control the conditional coverage error at individual test inputs. As a consequence, CP can remain unreliable in regions where the conditional score distributions are mismatched.
In this work, we bound the conditional invalidity of CP under distribution shift in terms of the Wasserstein distance between the calibration and test distributions. This result highlights the role of invertible transport in mitigating conditional coverage degradation. Motivated by this insight, we introduce Branched Normalizing Flow (BNF), a two-branch architecture that normalizes a test input to the calibration distribution and transforms the prediction set of the normalized input back to the test distribution while preserving conditional guarantees. Empirically, BNF consistently improves conditional coverage robustness on nine datasets across a wide range of confidence levels.
\end{abstract}

\section{Introduction}
Due to data noise and lack of prior knowledge, prediction uncertainty hinders applications of AI in various safety-critical domains. Conformal Prediction (CP) yields a set of possible targets rather than a single prediction to accommodate prediction uncertainty~\cite{vovk2005algorithmic, shafer2007tutorialconformalprediction}. We focus on CP for $\textbf{regression}$~\cite{lei2017distributionfreepredictiveinferenceregression}. Given a trained model $h$, a score function $s(X,Y) = |h(X) - Y|$ computes the residuals (conformal scores) of $n$ calibration instances $\{(X_i,Y_i)\}_{i=1}^n$. Denoting $\tau$ the ${\lceil(1-\alpha)(n+1)\rceil}/{n}$ quantile of the conformal scores, a vanilla prediction set $\mathcal{C}_\text{M}(X_{n+1})$ of a test input $X_{n+1}$ contains all target values whose conformal scores are smaller than $\tau$. Let $P_{XY}$ and $Q_{XY}$ be calibration and test distributions in space $\mathcal{X}\times\mathcal{Y}$, respectively. If the data are independent and identically distributed (i.i.d.) so that $P_{XY}= Q_{XY}$, the prediction set $\mathcal{C}_\text{M}(X_{n+1})$ achieves the \textbf{marginal coverage guarantee} $\text{Pr}\left(Y_{n+1}\in \mathcal{C}_\text{M}(X_{n+1})\right)\geq 1-\alpha$. However, since $\tau$ does not depend on the specific test input $x$, $\mathcal{C}_\text{M}(X_{n+1})$ has constant size and lacks adaptiveness. To address the weakness, conditional prediction set $\mathcal{C}_\text{C}(X_{n+1})$ aims at \textbf{conditional coverage guarantee} $\text{Pr}(Y_{n+1}\in \mathcal{C}_\text{C}(X_{n+1})|\colorbox{gray!20}{\raisebox{0pt}[1.2ex][0.1ex]{$X_{n+1}=x$}})\geq 1-\alpha, \forall x\in\mathcal{X}$, which provides more effective uncertainty quantification~\cite{papadopoulos2011regression,vovk2012conditional}.

In practice, a distribution shift ($P_{XY}\neq Q_{XY}$) can violate the i.i.d. assumption. For example, \textbf{multi-source domain generalization (MSDG)} considers $Q_{XY}$ as a random mixture of multiple source distributions~\cite{krueger2021out}. In this scenario, ensuring coverage guarantees becomes both important and challenging. 
Let $P_{V}$ and $Q_{V}$ be the calibration and test conformal score distributions in space $\mathcal{V}$, respectively. The difference between the cumulative probabilities of $P_{V}$ and $Q_{V}$ at $\tau$ can measure the validity of marginal coverage. 
Various upper bounds of the discrepancy between $P_V$ and $Q_V$ are proposed to estimate the potential deviation from the nominal marginal coverage~\cite{barber2023conformal, xu2025wassersteinregularized}.
Nevertheless, since existing methods align calibration and test conformal scores solely at the marginal level, they do not elucidate how score distributions change with individual inputs, leaving conditional score behavior largely unaddressed. As a result, these methods are unable to assess the conditional coverage of $\mathcal{C}_\text{C}(X_{n+1})$ under distribution shift (Figure~\ref{fig: Method}(a)).

\begin{figure*}[t]
\centering
\captionsetup{singlelinecheck = false, justification=justified}
% \vspace{-15pt}
  \includegraphics[scale=0.25]{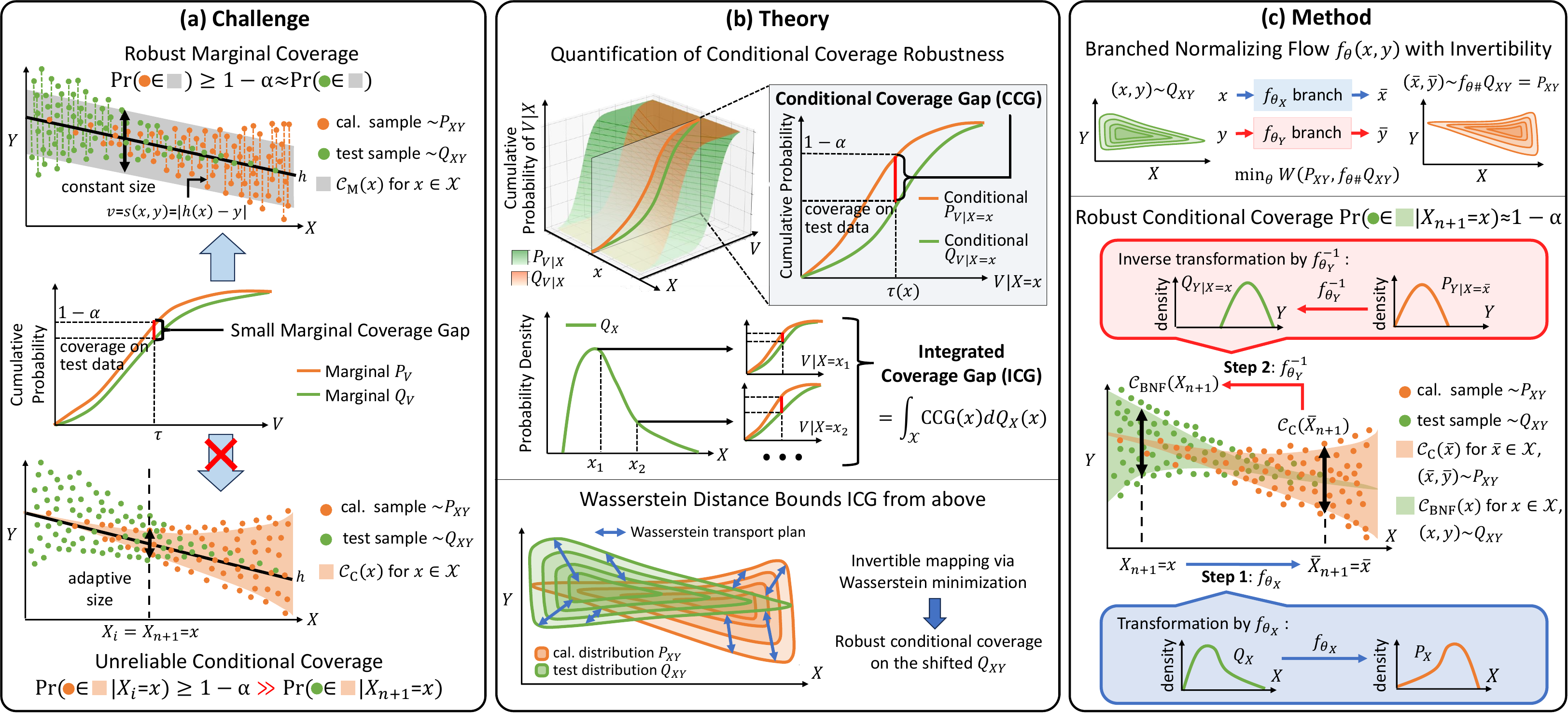}
  % \vspace{-3pt}
  \caption{\textbf{(a)} Vanilla prediction set $\mathcal{C}_\text{M}(x)$ has constant size and offers marginal coverage, which is robust if conformal score distributions $P_V$ and $Q_V$ have similar cumulative probabilities at $\tau$. Conditional prediction set $\mathcal{C}_\text{C}(x)$ has input-dependent size and provides conditional guarantees for calibration inputs $X_i=x$ where $i=1,...,n$, but may fail on non-i.i.d. test input $X_{n+1}=x$. The difference between $P_V$ and $Q_V$ can not capture the reliability of conditional coverage on the shifted test data; \textbf{(b)} Conditional coverage gap (CCG) measures $\mathcal{C}_\text{C}(X_{n+1})$ validity at $x$ by comparing $P_{V|X=x}$ and $Q_{V|X=x}$.  Integrated coverage gap (ICG) is the expectation of CCG under $Q_X$ for a holistic robustness measure. Wasserstein distance $W(P_{XY}, Q_{XY})$ bounds ICG to reveal how a distribution shift results in non-i.i.d. conformal scores. An invertible mapping between $P_{XY}$ and $Q_{XY}$ via Wasserstein minimization promotes robust conditional coverage; \textbf{(c)} Branched Normalizing Flow $f_\theta$ minimizes $W(f_{\theta\#}Q_{XY}, P_{XY})$, where $f_{\theta\#}Q_{XY}$ is a pushforward distribution.  For inference, we first compute a normalized test input $\overline{X}_{n+1}$ by $f_{\theta_X}$ and generate $\mathcal{C}_\text{C}(\overline{X}_{n+1})\subseteq\mathcal{Y}$ with conditional guarantee on $P_{XY}$. Then, $f_{\theta_Y}^{-1}$ inversely transforms the set to $\mathcal{C}_\text{BNF}(X_{n+1})\subseteq\mathcal{Y}$ with preserved conditional coverage on $Q_{XY}$.}
  \label{fig: Method} 
  \vspace{-10pt}
\end{figure*}

We aim to ensure the conditional guarantee under distribution shifts with three key contributions.
\begin{enumerate}[topsep=0pt, itemsep=0pt, leftmargin=15pt]
    \item \textbf{Quantification of conditional coverage robustness.} We define the conditional coverage gap (CCG) to measure conditional invalidity under distribution shift, and the integrated coverage gap (ICG) as the expectation of CCG under the test distribution. (Figure~\ref{fig: Method}(b), \nth{1} plot).
    \item \textbf{Bounding by Wasserstein distance.} We bound ICG with the Wasserstein distance between calibration and test distributions to reveal how a distribution shift is propagated from data space to score space. This bound implies that an invertible mapping can promote robust conditional coverage (Figure~\ref{fig: Method}(b) \nth{2} plot).
    \item \textbf{Branched Normalizing Flow (BNF).}
    We embed the Wasserstein bound into a branched structure, defined as $f_\theta(x,y) = (f_{\theta_X}(x), f_{\theta_Y}(y))=(\widebar{x},\widebar{y})$, to transform $Q_{XY}$ to $P_{XY}$ (Figure~\ref{fig: Method}(c) \nth{1} plot). The structure does not explicitly couple the transformations of $x$ and $y$, so $f_{\theta_X}$ can compute the normalized test input without knowing the true label during inference. If the conditional prediction set of the normalized input holds $1-\alpha$ conditional coverage on calibration distribution, $f_{\theta_Y}^{-1}$ inversely transforms it with preserved guarantee on test distribution (Figure~\ref{fig: Method}(c) \nth{2} plot). 
\end{enumerate}

Experiments on nine datasets cover both synthetic distribution shifts~\cite{physicochemical_properties_of_protein_tertiary_structure_265} and real-world challenges, including sales prediction across time series~\cite{bike_sharing_275}, traffic forecasting with mismatched data~\cite{cui2019traffic}, medicine decision-making for different populations~\cite{johnson2023mimic,pollard2018eicu}, and epidemic modeling over pandemic intervals~\cite{deng2020cola}, demonstrating improved conditional guarantee robustness under distribution shift.

\section{Background}\label{sec: background}
\subsection{Conditional Conformal Prediction}\label{sec: adaptive CP}
Denote $X\in\mathcal{X}\subseteq\mathbb{R}^d$ and $Y\in\mathcal{Y}\subseteq\mathbb{R}$ the input and output random variables, respectively. With a regression model $h: \mathcal{X}\rightarrow\mathcal{Y}$, a score function $s:\mathcal{X}\times\mathcal{Y}\rightarrow\mathcal{V}\subseteq\mathbb{R}$ outputs conformal scores to assess how data conform to the model $h$. We denote $V\in\mathcal{V}$ the random variable of conformal score, typically defined as the absolute residual: $V:=s(X, Y)=|h(X)-Y|$. With instances $\{(X_i, Y_i)\}_{i=1}^n$ from a calibration distribution $P_{XY}$, split conformal prediction computes calibration conformal scores $V_i=s(X_i,Y_i)$ for $i=1,...,n$~\citep{papadopoulos2002inductive}. For a test instance $(X_{n+1},Y_{n+1})$, a vanilla prediction set is given by $\mathcal{C}_\text{M}(X_{n+1}):=\left\{y:s(X_{n+1},y)\leq\tau,y\in\mathcal{Y}\right\}$, where $\tau$ is the ${\lceil(1-\alpha)(n+1)\rceil}/{n}$ quantile of $\{V_i\}_{i=1}^n$.\footnote{Equivalently, $\tau$ can be defined as the $1-\alpha$ quantile of $\{V_i\}_{i=1}^n\cup\{V_\infty\}$~\citep{vovk2005algorithmic,lei2017distributionfreepredictiveinferenceregression}.} Under the i.i.d. assumption with $(X_{n+1},Y_{n+1})\sim P_{XY}$, $\mathcal{C}_\text{M}(X_{n+1})$ provides $\textbf{marginal coverage guarantee}$ of the ground truth $Y_{n+1}$, namely,
\begin{equation}\label{eq: marginal guarantee}
    \text{Pr}\left(Y_{n+1}\in \mathcal{C}_\text{M}(X_{n+1})\right)=\text{Pr}(V_{n+1}\leq\tau)\geq 1-\alpha.
\end{equation} 
As $\tau$ is independent from test inputs, fixed-size prediction sets often underestimate uncertainty for hard samples and overestimate it for easy ones~\citep{angelopoulos2022conformal}. Therefore, a conditional prediction set $\mathcal{C}_\text{C}(X_{n+1})$ aims at improving the guarantee under the condition where $X_{n+1}=x$, $\forall x\in\mathcal{X}$. Theoretically, denote $\tau(x)$ the $(1-\alpha)$ quantile of $P_{V|X=x}$. Then, for $X_{n+1}=x$, a conditional prediction set is
\begin{equation} \label{eq: conditional prediction set}
    \mathcal{C}_\text{C}(X_{n+1}):=\left\{y:s(X_{n+1},y)\leq\tau(x), y\in\mathcal{Y}\right\},
\end{equation}
with the \textbf{conditional guarantee} under i.i.d. assumption:
\begin{equation}\label{eq: conditional guarantee}
        \text{Pr}\left(Y_{n+1}\in \mathcal{C}_\text{C}(X_{n+1})|X_{n+1}=x\right)=\text{Pr}\left(V_{n+1}\leq \tau(x)|X_{n+1}=x\right)\geq 1-\alpha, \forall{x}\in\mathcal{X}.
\end{equation}
\subsection{Conformal Prediction in Multi-Source Domains}
Multi-source domain generalization (MSDG) is a case of joint distribution shifts, where  $(X_{n+1},Y_{n+1})\sim Q_{XY}\neq P_{XY}$.
Under MSDG, the test distribution is a random mixture of multiple sources, so standard conformal prediction cannot maintain its coverage guarantees without accounting for the shift.
To address marginal coverage under such shifts,
conservative CP approaches consider the worst-case scenario~\citep{cauchois2024robust,zou2024coverage}. Recent work further regularizes the model $h$ for a balance between marginal coverage and prediction set size~\citep{xu2025wassersteinregularized}. A related area is federated CP~\citep{lu2023federated,wen2025distributedconformalpredictionmessage},  where robust CP is pursued across separated sources without data centralization. Nevertheless, the impact of joint distribution shifts on conditional coverage guarantees remains poorly understood. We therefore develop a theoretical framework to analyze conditional coverage robustness under such shifts.
\section{Theory}
\subsection{Conditional Coverage Lower Bound}
Let $P_{V|x}$ and $Q_{V|x}$ be the calibration and test conformal score distributions conditioned on an input $x$. Denote $F^P_{V|x}(\cdot)$ and $F^Q_{V|x}(\cdot)$ cumulative distribution functions (CDFs) of $P_{V|x}$ and $Q_{V|x}$, respectively. Using this notation, the conditional guarantee under i.i.d. assumption in Eq.~(\ref{eq: conditional guarantee}) is reformulated by
\begin{equation}\label{eq: conditional coverage under i.i.d. in cdf}
        F^P_{V|x}(\tau(x))\geq 1-\alpha.
\end{equation}
To quantify how a distribution shift $(X_{n+1},Y_{n+1}) \sim Q_{XY}\neq P_{XY}$ impedes the conditional guarantee at $X_{n+1}=x$, we define \textbf{conditional coverage gap (CCG)} by
\begin{equation}\label{eq: conditional coverage gap}
    \text{CCG}(P,Q,x):=\big|F^P_{V|x}(\tau(x))-F^Q_{V|x}(\tau(x))\big|,
\end{equation}
which implies $F^Q_{V|x}(\tau(x))\geq F^P_{V|x}(\tau(x))-\text{CCG}(P,Q,x)$. Combining with  Eq.~(\ref{eq: conditional coverage under i.i.d. in cdf}), we obtain a lower bound on the conditional coverage for $Y_{n+1}\sim Q_{Y|x}$ by
\begin{equation}\label{eq: lowerbound by CCG}\begin{split}
        &\Pr\left(Y_{n+1}\in \mathcal{C}_\text{C}(X_{n+1})|X_{n+1}=x\right)=F^Q_{V|x}(\tau(x))\\&\geq F^P_{V|x}(\tau(x))-\text{CCG}(P,Q,x)\geq 1-\alpha-\text{CCG}(P,Q,x).
\end{split}
\end{equation}
Since test inputs follow $Q_X$, evaluating CCG at a single $x$ can not take $Q_X(x)$ at different $x$ into account. Hence,  \textbf{integrated coverage gap (ICG)} is defined as the expectation of CCG under $Q_X$ by
\begin{equation}\label{eq: integrated coverage gap}
    \text{ICG}(P,Q) := \int_\mathcal{X} \text{CCG}(P,Q,x) \dd Q_X(x).
\end{equation}
By integrating CCG over $Q_X$, ICG is a comprehensive metric for the validity of $\mathcal{C}_\text{C}(X_{n+1})$. A low ICG means that conditional coverage is consistently close to $1-\alpha$ over $Q_{X}$.
\subsection{Linking Data Shift to Coverage via Wasserstein}
We further explore how a distribution shift between $P_{XY}$ and $Q_{XY}$ in space $\mathcal{X}\times\mathcal{Y}$ is propagated to a shift between $P_{V|x}$ and $Q_{V|x}$ in space $\mathcal{V}$ for all $x\in\mathcal{X}$. 
\begin{definition}[$p$-Wasserstein Distance between Population Distributions~\citep{panaretos2019statistical}]
    For any probability measures $\mu_X$ and $\nu_X$ defined on a metric space $(\mathcal{X},d_\mathcal{X})$, where $\mathcal{X}$ is a set and $d_\mathcal{X}$ is a metric on $\mathcal{X}$, the Wasserstein distance of order $p\geq 1$ between $\mu_X$ and $\nu_X$ is defined by
    \begin{equation*}\label{eq: Wasserstein distance}
        {W_p}(\mu_X,\nu_X)\shorteq{0.1em}\inf_{{\gamma\in\Gamma(\mu_X,\nu_X)}} \Bigl(\int\nolimits_{\mathcal{X}\times\mathcal{X}}{d_\mathcal{X}(x_1,x_2)^p\dd{\gamma(x_1,x_2)}}\Bigr)^\frac{1}{p},
    \end{equation*}
    where $\Gamma(\mu_X,\nu_X)$ is the set of all joint probability measures $\gamma$ on $\mathcal{X}\times\mathcal{X}$ with marginals $\gamma(\mathcal{A}\times\mathcal{X})=\mu_X(\mathcal{A})$ and $\gamma(\mathcal{X}\times \mathcal{B})=\nu_X(\mathcal{B})$, $\forall$ measurable sets $\mathcal{A},\mathcal{B}\subseteq\mathcal{X}$.
\end{definition}
The Wasserstein distance with $p=1$ is denoted as $W$. An upper bound of the marginal coverage gap is proposed in~\citep{xu2025wassersteinregularized}. Let $L$ be the Lebesgue density bound of $P_{V|x}$ for all $x\in\mathcal{X}$~\citep{ross2011fundamentals}. We derive
\begin{equation}\label{eq: conditional gap upper bound V}
    \text{CCG}(P,Q,x)\leq\sqrt{2L\cdot W(P_{V|x},Q_{V|x})}.
\end{equation}
Next, we explore how $W(P_{V|x},Q_{V|x})$ arises from the difference in $P_{Y|x}$ and $Q_{Y|x}$ by Theorem~\ref{theorem: Wasserstein distance with Lipschitz continuity}.
\begin{theorem}\label{theorem: Wasserstein distance with Lipschitz continuity}
    Let $\mu_{XY}$ and $\nu_{XY}$ be probability measures in the metric space $(\mathcal{X}\times\mathcal{Y}, d_\mathcal{XY})$, where $d_{\mathcal{X}\mathcal{Y}}$ is the 2-product metric of $d_\mathcal{X}$ and $d_\mathcal{Y}$ such that $d_{\mathcal{X}\mathcal{Y}}((x_1,y_1),(x_2,y_2)):=||(d_\mathcal{X}(x_1,x_2),d_\mathcal{Y}(y_1,y_2))||_2$. Let $s: \mathcal{X}\times\mathcal{Y} \to \mathcal{V}$ be a measurable function such that $s(x,y)=v$. In the metric space $(\mathcal{V}, d_\mathcal{V})$, denote $\mu_V$ the probability measure of $s(X,Y)$ for $(X,Y) \sim \mu_{XY}$. Also, let $\nu_V$ be the probability measure of $s(X,Y)$ for $(X,Y) \sim \nu_{XY}$. If $s$ has a continuity constant $\kappa$ at $x$ such that $\frac{d_\mathcal{V}(s(x,y_1),s(x,y_2))}{d_\mathcal{Y}(y_1,y_2)}\leq\kappa, \forall x\in\mathcal{X}$ and $\forall y_1,y_2\in\mathcal{Y}$, the following inequality holds:
    \begin{equation}
        W(\mu_{V|x},\nu_{V|x})\leq \kappa\cdot W(\mu_{Y|x},\nu_{Y|x}).
    \end{equation}
\end{theorem}
A related theorem in~\citep{xu2025wassersteinregularized} does not condition on a specific $x$. Since $\mathcal{V}, \mathcal{Y} \subseteq \mathbb{R}$, we can take the metrics $d_\mathcal{V}(\cdot, \cdot)$ and $d_\mathcal{Y}(\cdot, \cdot)$ as the absolute value of the difference. Therefore, according to Theorem~\ref{theorem: Wasserstein distance with Lipschitz continuity}, if the score function $s(X,Y)$ is continuous with a constant $\kappa$ such that $\frac{|s(x,y_1)-s(x,y_2)|}{|y_1-y_2|}\leq \kappa$, $\forall x\in \mathcal{X},\forall y_1,y_2\in\mathcal{Y}$, we derive that
\begin{equation}\label{eq: kappa inequality}
    W(P_{V|x},Q_{V|x})\leq\kappa\cdot W(P_{Y|x},Q_{Y|x}).
\end{equation}
For an intuitive explanation, a smaller $\kappa$ implies that the score function $s$ becomes less responsive to changes in $y$ conditioned on $x$. Consequently, a substantial distribution shift between $P_{Y|x}$ and $Q_{Y|x}$ will not result in a large $W(P_{V|x}, Q_{V|x})$. Combining Eq.~(\ref{eq: kappa inequality}) and Eq.~(\ref{eq: conditional gap upper bound V}) leads to
\begin{equation}\label{eq: conditional gap upper bound Y}
        \text{CCG}(P,Q,x)\leq\sqrt{2\kappa L\cdot W(P_{Y|x},Q_{Y|x})}.
\end{equation}
Besides, as $\sqrt{W(P_{Y|x},Q_{Y|x})}\leq W(P_{Y|x},Q_{Y|x})+1/4$,\footnote{When $W(P_{Y|x},Q_{Y|x})\geq 1$, we obtain $\text{ICG}(P,Q)\leq\sqrt{2\kappa L}\int_\mathcal{X} W(P_{Y|x}, Q_{Y|x})\dd Q_X(x)$ by $\sqrt{W(P_{Y|x},Q_{Y|x})}\leq W(P_{Y|x},Q_{Y|x})$. This tightens Eq.~(\ref{eq: final bound of ICG}) to $\text{ICG}(P,Q)\leq\sqrt{2\kappa L}\cdot\eta\cdot W(P_{XY},Q_{XY})$}. We can bound ICG based on Eq.~(\ref{eq: conditional gap upper bound Y}) by 
\begin{equation*}
    \text{ICG}(P,Q)\leq\sqrt{2\kappa L}\left(\int_\mathcal{X} W(P_{Y|x}, Q_{Y|x})\dd Q_X(x)+\frac{1}{4}\right),
\end{equation*}
showing that transporting 
$P_{Y|x}$ to $Q_{Y|x}$ for all $x\in\mathcal{X}$ at population level is sufficient to eliminate conditional gap. 

However, accurately estimating the conditional distributions $P_{Y|x}$ and $Q_{Y|x}$ from finite samples is theoretically intractable in practice, particularly in moderate or high dimensions. To address this, we further bound $\text{ICG}(P, Q)$ via the joint Wasserstein distance $W(P_{XY}, Q_{XY})$ to provide a theoretically justified and more tractable alternative.

\begin{theorem}\label{theorem: upper bound of conditional Wasserstein distance}
    Let $\mu_{XY}$ and $\nu_{XY}$ be probability measures on the metric space $(\mathcal{X}\times\mathcal{Y}, d_{\mathcal{X}\mathcal{Y}})$. $\mu_{Y|x}$ and $\nu_{Y|x}$ are the corresponding conditional distributions of $Y$ given $X=x$. A joint distribution shift occurs between $\mu_{XY}$ and $\nu_{XY}$ such that $\mu_{X}\neq\nu_{X}$ and $\mu_{Y|X}\neq\nu_{Y|X}$. Denote $\gamma^*_{XYXY}\in\Gamma(\mu_{XY},\nu_{XY})$ the optimal transport plan of $W(\mu_{XY},\nu_{XY})$ and $\gamma^*_{XX}(x_1,x_2)=\int_{\mathcal{Y}^2}\dd\gamma^*_{XYXY}(x_1,y_1,x_2,y_2)$. If $\exists$ $\eta>0$ such that $\int_{\mathcal{X}}W(\mu_{Y|x},\nu_{Y|x})\dd\nu_X(x)\leq\eta\int_{\mathcal{X}\times\mathcal{X}}W(\mu_{Y|x},\nu_{Y|x})\dd\gamma^*_{XX}(x,x)$, the following inequality holds that $\int_{\mathcal{X}}W(\mu_{Y|x},\nu_{Y|x})\dd\nu_X(x)\leq\eta\cdot W(\mu_{XY},\nu_{XY})$.
\end{theorem}
Substituting the notations $\mu$ and $\nu$ with $P$ and $Q$ in Theorem~\ref{theorem: upper bound of conditional Wasserstein distance}, we establish an upper bound for the integrated conditional Wasserstein distance as follows
\begin{equation}\label{eq: joint W-distance upper bound}
    \int_\mathcal{X} W(P_{Y|x}, Q_{Y|x})\dd Q_X(x)\leq \eta\cdot W(P_{XY},Q_{XY}).
\end{equation}
Finally, based on Eq.~(\ref{eq: joint W-distance upper bound}), we deduce that
\begin{equation}\label{eq: final bound of ICG}
    \text{ICG}(P,Q)\leq\sqrt{2\kappa L}\ (\eta\cdot W(P_{XY},Q_{XY})+1/4).
\end{equation}
Eq.~(\ref{eq: final bound of ICG}) states that ICG is bounded by $W(P_{XY}, Q_{XY})$, meaning that greater shifts in the joint distribution lead to a more significant decline in conditional coverage. However, the influence of $W(P_{XY}, Q_{XY})$ is moderated by scaling constants, which include $\kappa$, $L$, and $\eta$. The specific roles and particular implications of these constants for CP are detailed in Appendix~\ref{appendix: scaling constants}. The finite-sample behavior of $W(P_{XY}, Q_{XY})$ is examined in Appendix \ref{appendix: finite sample analysis}.
\section{Method}\label{sec: method}
The upper bound in Eq.~(\ref{eq: final bound of ICG}) provides a framework to ensure conditional coverage under distribution shift. Specifically, if a model $f_\theta$ transforms $Q_{XY}$ via the Wasserstein transport plan to $P_{XY}$, for $(X_{n+1},Y_{n+1})\sim Q_{XY}$, we have
\begin{equation}\label{eq: exact alignment}(\widebar{X}_{n+1},\widebar{Y}_{n+1}):=f_\theta(X_{n+1},Y_{n+1})\sim P_{XY}. 
\end{equation}
Therefore, the conditional prediction set $\mathcal{C}_\text{C}(\widebar{X}_{n+1})$ constructed on the normalized input ensures conditional coverage with respect to $P_{XY}$. However, to achieve $1-\alpha$ conditional coverage on $Q_{XY}$ during inference, the model $f_\theta$ must satisfy two additional requirements.
\begin{enumerate}[label=(\roman*), topsep=0pt, itemsep=0pt, leftmargin=15pt]
\item To obtain a prediction set of the original test input $X_{n+1}$ from $\mathcal{C}_\text{C}(\widebar{X}_{n+1})$, $f_\theta$ should be invertible:\label{itm:1}
\begin{equation}
   f_\theta^{-1}(\widebar{X}_{n+1},\widebar{Y}_{n+1})=(X_{n+1},Y_{n+1}).
\end{equation}
\item Since $Y_{n+1}$ is unobserved at inference, $f_\theta$ should not depend on $Y_{n+1}$ when transform of $X_{n+1}$:\label{itm:2}
\begin{equation}
    \widebar{X}_{n+1}\perp Y_{n+1}|X_{n+1}.
\end{equation}
\end{enumerate}
\subsection{Branched Normalizing Flow}
Normalizing flows are widely applied techniques for invertible mapping~\citep{kobyzev2020normalizing, papamakarios2021normalizing}. A formal definition of normalizing flows is presented in Definition~\ref {def: normalizing flow} with a demonstration in Figure~\ref{fig: pushforward}.
\begin{definition}[Normalizing flows~\citep{kobyzev2020normalizing}]\label{def: normalizing flow}
    Let $\mu_X$ be a probability measure in $\mathbb{R}^d$. For a measurable and invertible function $g: \mathbb{R}^d \rightarrow \mathbb{R}^d$, $\nu_X$ is the pushforward measure of $\mu_X$ through $g$, denoted as $\nu_X=g_\#\mu_X$, if $\nu_X(\mathcal{A}) = \mu_X(g^{-1}(\mathcal{A}))$ for every measurable set $\mathcal{A} \subseteq \mathbb{R}^d$. $g$ is referred to as the generative flow, and $f = g^{-1}$ is known as the normalizing flow with $\mu_X=f_\#\nu_X$.
\end{definition}
\begin{wrapfigure}[7]{r}{7.2cm}
\centering
\captionsetup{singlelinecheck = false, skip=5pt, justification=centering}
\vspace{-12pt}
  \includegraphics[width=0.5\textwidth]{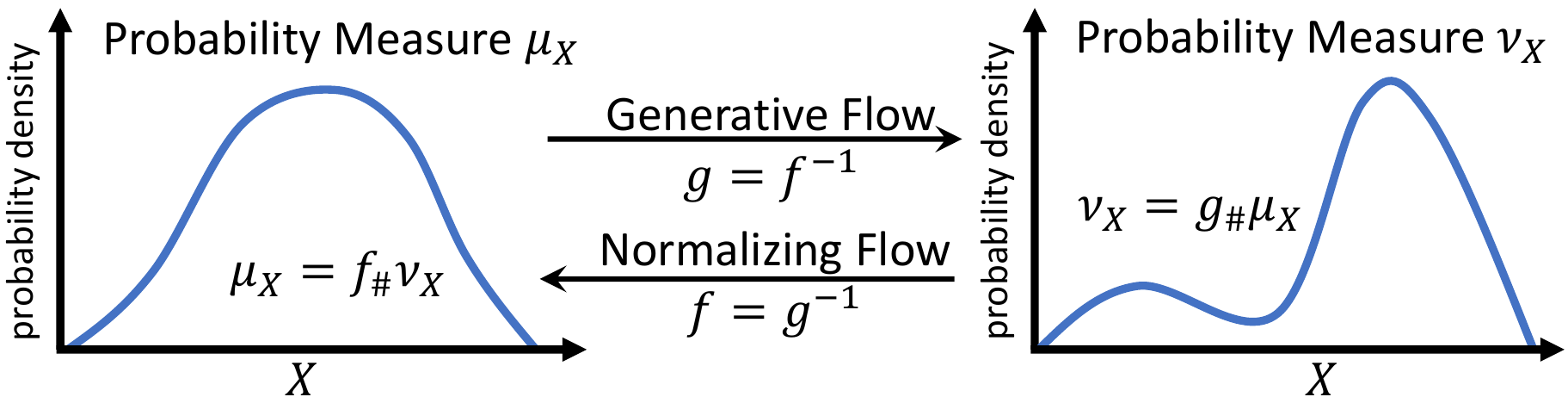}
  \vspace{-2pt}
  \caption{
  Invertible generative and normalizing flows.}
  \label{fig: pushforward} 
\end{wrapfigure}
To make $f_\theta$ meet the two requirements~\ref{itm:1} and~\ref{itm:2}, we introduce a special normalizing flow, called \textbf{Branched Normalizing Flow (BNF)}. For a given sample $(x,y)$, BNF transforms it with a branched structure such that 
\begin{equation}\label{eq: original BNF}
    (\widebar{x}, \widebar{y}):\shorteq{0.1em}f_\theta(x,y) \shorteq{0.1em} (f_{\theta_X}(x), f_{\theta_Y}(y)).
\end{equation}
The invertibility of BNF allows that $
    f_\theta^{-1}(\widebar{x}, \widebar{y}) = (f^{-1}_{\theta_X}(\widebar{x}), f^{-1}_{\theta_Y}(\widebar{y})) = (x, y),$
enabling the inverse transformation of $\mathcal{C}_\text{C}(\widebar{X}_{n+1})$ and satisfying requirement~\ref{itm:1}. Besides, the parameters $\theta_X$ and $\theta_Y$ are not shared between branches, so BNF does not explicitly couple the mappings of $x$ and $y$. Therefore, the normalized test input $\widebar{X}_{n+1}$ can be obtained without knowing 
 $Y_{n+1}$, fulfilling requirement~\ref{itm:2}. 

Consider a BNF realizing $f_{\theta\#}Q_{XY}\shorteq{0.1em}P_{XY}$ so that $f_{\theta_{X}\#}Q_X\shorteq{0.1em}P_X$ and $ f_{\theta_{Y}\#}Q_{Y|x}\shorteq{0.1em}P_{Y|\widebar{x}}$ by optimizing
\begin{equation}\label{eq: joint Wasserstein minimization}
    \min_\theta W(P_{XY},f_{\theta\#}Q_{XY}).
\end{equation}
Then, given a test input $X_{n+1}=x$, we normalize it as $\widebar{X}_{n+1}:=f_{\theta_X}(x)=\widebar{x}$. Since $f_{\theta\#}Q_{XY} = P_{XY}$, the transformed true target $\widebar{Y}_{n+1}:=f_{\theta_Y}({Y}_{n+1})$, together with $\widebar{X}_{n+1}$, should follow the calibration distribution, i.e, $(\widebar{X}_{n+1}, \widebar{Y}_{n+1})\sim P_{XY}$, as shown in Figure~\ref{fig: Method}(c) \nth{1} plot. Therefore, the adaptive prediction set of $\widebar{X}_{n+1}$ ensures 
% Then, as shown in Figure~\ref{fig: Method}(c), given a test input $X_{n+1}=x$, we normalize it as $\widebar{X}_{n+1}=f_{\theta_X}(x)=\widebar{x}$ and construct an adaptive prediction set $\mathcal{C}_\text{C}(\widebar{X}_{n+1})$ using $\tau(\widebar{x})$ as in Eq.~(\ref{eq: adaptive prediction set}). Since $f_{\theta\#}Q_{XY} = P_{XY}$, the transformed true target $\widebar{Y}_{n+1}=f_{\theta_Y}({Y}_{n+1})$, together with $\widebar{X}_{n+1}$, should follow the calibration distribution, i.e, $(\widebar{X}_{n+1}, \widebar{Y}_{n+1})\sim P_{XY}$.
% Therefore, by Eq.~(\ref{eq: conditional guarantee}), we have the conditional guarantee: 
\begin{equation}\label{eq: conditional guarantee of the normalized sample}
    \text{Pr}\left(\widebar{Y}_{n+1}\in \mathcal{C}_\text{C}(\widebar{X}_{n+1})|\widebar{X}_{n+1}=\widebar{x}\right)\geq 1-\alpha.
\end{equation}
BNF then constructs a prediction set of the original input $X_{n+1}$ by including all targets whose normalized counterparts lie in $\mathcal{C}_\text{C}(\widebar{X}_{n+1})$. Specifically, we define 
\begin{equation}\label{eq: BNF prediction set}
    \mathcal{C}_\text{BNF}(X_{n+1}):=\{ f_{\theta_Y}^{-1}(\widebar{y}): \widebar{y}\in \mathcal{C}_\text{C}(\widebar{X}_{n+1})\}.
\end{equation}
Proposition~\ref{proposition: equivalence with monoticity} in Appendix~\ref{appendix: additional theoretical statements} and the invertibility of the univariate function $f_{\theta_Y}$ imply that
\begin{equation}\label{eq: equivalence with monoticity}
        \widebar{Y}_{n+1}\in \mathcal{C}_\text{C}(\widebar{X}_{n+1})\iff f^{-1}_{\theta_Y}(\widebar{Y}_{n+1})\in \mathcal{C}_\text{BNF}({X}_{n+1}).
\end{equation}
Consequently, since $Y_{n+1}=f^{-1}_{\theta_Y}(\widebar{Y}_{n+1})$, the conditional guarantee is inherited by $\mathcal{C}_\text{BNF}(X_{n+1})$:
\begin{equation}\label{eq: BNF conditional guarantee}
    \text{Pr}\left(Y_{n+1}\in \mathcal{C}_\text{BNF}(X_{n+1})|X_{n+1}=x\right)\geq 1-\alpha.
\end{equation}
Even if $f_{\theta_X}$ and $f_{\theta_Y}$ do not share parameters, Wasserstein minimization in Eq.~(\ref{eq: joint Wasserstein minimization}) considers the dependency between them. We provide an illustrative example in Appendix~\ref{appendix: demonstration of implicit dependency} to show how this dependency is implicitly accounted for during optimization.
\subsection{Enhancing Fitting Ability via Gaussian Noise Augmentation}\label{sec: augmented BNF}
The monotonicity of the univariate $f_{\theta_Y}$ allows the equivalence in Eq.~(\ref{eq: equivalence with monoticity}), but also limits its fitting ability. As a result, it struggles to optimize Eq.~(\ref{eq: joint Wasserstein minimization}) for complex distributions, leading to unreliable conditional coverage. We empirically present the issue in Appendix~\ref{appendix: comparison between flows}.

To address this limitation, we adopt the augmentation technique proposed in~\citep{huang2020augmentednormalizingflowsbridging} and introduce a variant called \textbf{Augmented BNF} to gain higher fitting ability. Specifically, given a sample $(x,y)$, the augmented transformation is defined as
\begin{equation}\label{eq: normalization via augmented BNF}
        (\widebar{x}, \widebar{y}):= (f_{\theta_X}(x), f_{\theta_Y}^\text{aug}(y;\varepsilon)).
\end{equation}
where $\varepsilon$ is sampled from a Gaussain distribution $\mathcal{N}(0,1)$. Meanwhile, $f_{\theta_X}$ is unchanged from BNF. 
We implement Augmented BNF using Real NVP~\citep{dinh2016density, huang2020augmentednormalizingflowsbridging}, a representative coupling flow with architectural details provided in Appendix~\ref{appendix: architecture}.

Although $f^{\text{aug}}_{\theta_Y}(y;\varepsilon)$ remains invertible, it does not build a monotonic relationship between $y$ and $\bar{y}$. As a result, we can not rely on Proposition~\ref{proposition: equivalence with monoticity} to preserve the conditional guarantee. To address this issue, we propose an alternative approach to obtain a prediction set for test input $X_{n+1}$ with a sampled noise $\varepsilon_{n+1}$ by defining
\begin{equation}\label{eq: aug BNF prediction set}
    \mathcal{C}_\text{BNF}^{\text{aug}}(X_{n+1}):\text{}\shorteq{0.1em}\left\{y:f^{\text{aug}}_{\theta_Y}(y;\varepsilon_{n+1})\in \mathcal{C}_\text{C}(\widebar{X}_{n+1})\right\}.
\end{equation}
Proposition~\ref{proposition: equivalence without monoticity} in Appendix~\ref{appendix: additional theoretical statements} implies that $\widebar{Y}_{n+1}\in \mathcal{C}_\text{C}(\widebar{X}_{n+1})\iff Y_{n+1}\in \mathcal{C}_\text{BNF}^\text{aug}({X}_{n+1})$. Hence, based on Eq.~(\ref{eq: conditional guarantee of the normalized sample}), we conclude that
\begin{equation}\label{eq: BNF conditional guarantee with augmentataion}
    \text{Pr}\left(Y_{n+1}\in \mathcal{C}_\text{BNF}^\text{aug}(X_{n+1})|X_{n+1}=x\right)\geq 1-\alpha.
\end{equation}
\section{Application to Multi-Source Domains}\label{sec: application to MSDG}
In this work, we study joint distribution shift in multi-source domain generalization (MSDG)~\citep{sagawa2019distributionally}, a widely explored setting in CP~\citep{cauchois2024robust,zou2024coverage,xu2025wassersteinregularized}. In MSDG, the test distribution is a random mixture within the convex hull of the source distributions. Formally, given $K$ source distributions $D_{XY}^{k}$ for $k\shorteq{0.1em}1,..,K$, the test distribution satisfies $
    Q_{XY}\in\left\{\sum\nolimits_{k=1}^K\lambda_kD_{XY}^k: \lambda_1,...,\lambda_K\geq 0,\sum\nolimits_{k=1}^K\lambda_k=1\right\}$.
\begin{theorem}\label{theorem: Wasserstein upper bound under random mixture}
    Let $\{\nu^k_{XY}\}_{k=1}^K$ be probability measures defined on the metric space $(\mathcal{X} \times \mathcal{Y}, d_{\mathcal{XY}})$, and let $\nu_{XY}$ lie in the convex hull of these measures, i.e., $\nu_{XY} = \sum\nolimits_{k=1}^K \lambda_k \nu^k_{XY}$ with $\lambda_k \geq 0$ and $\sum\nolimits_{k=1}^K \lambda_k = 1$. For any probability measure $\mu_{XY}$ on $(\mathcal{X} \times \mathcal{Y}, d_{\mathcal{XY}})$, the following inequality holds: $        W(\mu_{XY},\nu_{XY})\leq\sum\nolimits_{k=1}^K \lambda_kW(\mu_{XY},\nu_{XY}^k)$.    
\end{theorem}
As outlined in~\citep{cauchois2024robust,xu2025wassersteinregularized}, achieving coverage guarantee for each source distribution ensures that the coverage on the test distribution is preserved. Inspired by the principle, Theorem~\ref{theorem: Wasserstein upper bound under random mixture} suggests a surrogate objective for Augmented BNF by $\sum\nolimits_{k=1}^K \lambda_kW(P_{XY},{f_{\theta}^\text{aug}}_\#D_{XY}^k)$.  
Since $\{\lambda_k\}_{k=1}^K$ are unknown, we minimize the expectation assuming they are uniformly distributed over the simplex:
\begin{equation}\label{eq: joint Wasserstein minimization under random mixture}
    \min_\theta\tfrac{1}{K}\sum\nolimits_{k=1}^K W(P_{XY},{f_{\theta}^\text{aug}}_\#D^k_{XY}).
\end{equation}
\begin{wrapfigure}[25]{r}{7cm}
\vspace{-0.8cm}
\begin{minipage}{7cm}
\begin{algorithm}[H]
   \caption{Augmented BNF + CQR under MSDG}
   \label{alg: Augmented BNF under MSDG}
\begin{algorithmic}
    \STATE {\bfseries Require:} {sets $\mathcal{S}_{D^k}$ for $k\shorteq{0.1em}1,...,K$; calibration set $\mathcal{S}_{P}$; test set $\mathcal{S}_{Q}$; $N$ epochs; $1-\alpha$ confidence; Augmented BNF $f_\theta^\text{aug}$; CQR algorithm $A_\text{CQR}$.}
    \vspace{-0.1cm}
    \\\hrulefill
    \STATE {\bfseries Training Phase:}
\FOR{$i \shorteq{0.1em} 1$ to $N$ epochs}
    \FOR{$k \shorteq{0.1em} 1$ to $K$}
        \STATE Initialize $\widebar{\mathcal{S}}_{D^k} \gets \emptyset$
        \FOR{each $(x, y) \in \mathcal{S}_{D^k}$}
            \STATE $(\widebar{x}, \widebar{y})\shorteq{0.1em} f_{\theta}^{\text{aug}}(x, y, \varepsilon)$, where $\varepsilon\sim\mathcal{N}(0,1)$\STATE $\widebar{\mathcal{S}}_{D^k} \gets \widebar{\mathcal{S}}_{D^k} \cup \{(\widebar{x}, \widebar{y})\}$
        \ENDFOR
    \ENDFOR
    \STATE $\min_\theta \frac{1}{K} \sum_{k=1}^K W\left(\widehat{P}_{XY}, {f_{\theta}^{\text{aug}}}_\#\widehat{D}^k_{XY} \right)$
\ENDFOR
\vspace{-0.1cm}
    \\\hrulefill
    \STATE {\bfseries Inference Phase:}
    \FOR{$x\text{ from }\mathcal{S}_Q$}
    \STATE $\bar{x}\shorteq{0.1em}f_{\theta_X}(x)$
    \STATE $\mathcal{C}_\text{CQR}(\widebar{x})\shorteq{0.1em} A_\text{CQR}\left(\bigcup_{k=1}^K \mathcal{S}_{D^k}, \mathcal{S}_P, \widebar{x}, 1 - \alpha\right)$
    \STATE Sample $\varepsilon\sim\mathcal{N}(0,1)$ 
    \STATE $\mathcal{C}_\text{BNF}^\text{aug}(x)\shorteq{0.1em}\{y:f^{\text{aug}}_{\theta_Y}(y;\varepsilon)\in \mathcal{C}_\text{CQR}(\widebar{x})\}$
    \ENDFOR
\end{algorithmic}
\end{algorithm}
\end{minipage}
\end{wrapfigure}
We work with finite samples in practice. 
Let $\mathcal{S}_{D^k}$ denote a set of samples drawn from the $k$-th source distribution $D_{XY}^k$ for $k=1,...,n$
, each of equal size, and let $\mathcal{S}_P$ be a calibration set drawn from $P_{XY}$.
During optimization, for each $(x,y)\in\mathcal{S}_{D^k}$, we sample a noise $\varepsilon\sim\mathcal{N}(0,1)$ and compute $(\widebar{x},\widebar{y})$ using Eq.~(\ref{eq: normalization via augmented BNF}). All normalized pairs are collected in $\widebar{\mathcal{S}}_{D^k}$. 
The empirical distributions $\widehat{P}_{XY}$
and ${f_{\theta}^\text{aug}}_\#\widehat{D}^k_{XY}$ are estimated from $\mathcal{S}_P$ and $\widebar{\mathcal{S}}_{D^k}$, respectively, allowing us to optimize the objective in Eq.~(\ref{eq: joint Wasserstein minimization under random mixture}).

Even if $(\widebar{X}_{n+1},\widebar{Y}_{n+1})\sim P_{XY}$, constructing $\mathcal{C}_\text{C}(\widebar{X}_{n+1})$ that satisfies the conditional guarantee under $P_{XY}$ remains challenging with finite samples~\cite{foygel2021limits}. 
In this work, we employ conformalized quantile regression (CQR)~\citep{romano2019conformalized}, which generates a prediction set $\mathcal{C}_\text{CQR}(\widebar{X}_{n+1})$ to approximate the $1-\alpha$ conditional coverage in Eq.~(\ref{eq: conditional guarantee of the normalized sample}).
Crucially, CQR is independent of the Augmented BNF and can be seamlessly integrated into our framework. Details of the CQR implementation are provided in Appendix~\ref{appendix: CQR}.
We denote the algorithm of CQR as $
    A_\text{CQR}(\bigcup\nolimits_{k=1}^K \mathcal{S}_{D^k}, \mathcal{S}_P, \widebar{X}_{n+1}, 1 - \alpha)$.
Given a test set $\mathcal{S}_Q$ from $Q_{XY}$, we outline the combination of Augmented BNF + CQR in Algorithm~\ref{alg: Augmented BNF under MSDG}.
\vspace{-5pt}
\section{Experiment}\label{sec: experiment}
\vspace{-5pt}
\subsection{Experimental Setup}
We conduct Augmented BNF using the normflows library~\citep{Stimper2023}.
% The experiment environment is introduced in Appendix~\ref{appendix: exp environment}. 
To estimate the empirical Wasserstein distance, we adopt the Sinkhorn algorithm~\citep{cuturi2013sinkhorn,knight2008sinkhorn} via the geomloss library~\citep{feydy2019interpolating} in Appendix~\ref{appendix: sinkhorn algorithm}.

\textbf{Baselines.} Five methods are selected for comparison. Split CP (SCP)~\citep{papadopoulos2002inductive} ensures marginal coverage under i.i.d. data; Importance-Weighted CP (IW-CP)~\citep{tibshirani2019conformal} addresses covariate shift; Worst-Case CP (WC-CP)~\citep{cauchois2024robust,zou2024coverage,gendler2021adversarially} provides conservative guarantees under joint distribution shift; and Wasserstein-Regularized CP (WR-CP)~\citep{xu2025wassersteinregularized} improves robustness under MSDG. We also include CQR alone, without Augmented BNF, to highlight its limitations in pursuing conditional coverage under shift. Additional details about the baselines are given in Appendix~\ref{appendix: baselines}, with examples in Figure~\ref{fig: examples}.

\textbf{Datasets.} We set $K=3$ under both \textbf{synthetic} and \textbf{natural} distribution shifts. Synthetic shifts are introduced in the PTS dataset~\citep{physicochemical_properties_of_protein_tertiary_structure_265}. For real-world applications, we consider (i) sales prediction over time with Bike Rental~\citep{bike_sharing_275}, (ii) multi-location traffic forecasting with Seattle-Loop~\citep{cui2019traffic}, PEMSD4, and PEMSD8~\citep{bai2020adaptive}, (iii) unbiased healthcare with MIMIC-IV~\citep{johnson2023mimic}, eICU~\citep{pollard2018eicu}, and data from a collaborating hospital, and (iv) epidemic modeling across pandemic phases with U.S. Influenza-like Illness (ILI)~\citep{deng2020cola}. Data preparation for MSDG details are in Appendix~\ref{appendix: datasets}.

\textbf{Evaluation metric.} The worst-slice coverage (WSC)~\citep{cauchois2021knowing} measures the minimal coverage over any sufficiently large slice in $\mathcal{X}$, serving as an empirical proxy for the robustness of conditional coverage. However, as reviewed in Appendix~\ref{appendix: wsc}, WSC captures only the minimal (i.e., most insufficient) coverage and overlooks regions where coverage may be overly conservative. To address these weaknesses, we propose \textbf{worst-slice coverage gap (WSCG)} that captures both under- and over-coverage by taking the maximum absolute deviation from $1-\alpha$ over slices containing at least 10\% of test samples. Specifically, for any CP methods that produce a prediction set $\mathcal{C}(x)$ given an input $x$,
\begin{equation}\label{eq: WSCG}
    \text{WSCG}=\sup\nolimits_{\mathcal{S} \subseteq \mathcal{X}} \left|\Pr(y\in C(x)|x\in\mathcal{S})-(1-\alpha)\right|,\quad \text{s.t.} \Pr(x\in\mathcal{S}|(x,y)\in\mathcal{S}_Q)\geq 0.1.
\end{equation}
\begin{figure*}[t]
\centering
\captionsetup{singlelinecheck = false, justification=justified}
  \includegraphics[scale=0.4]{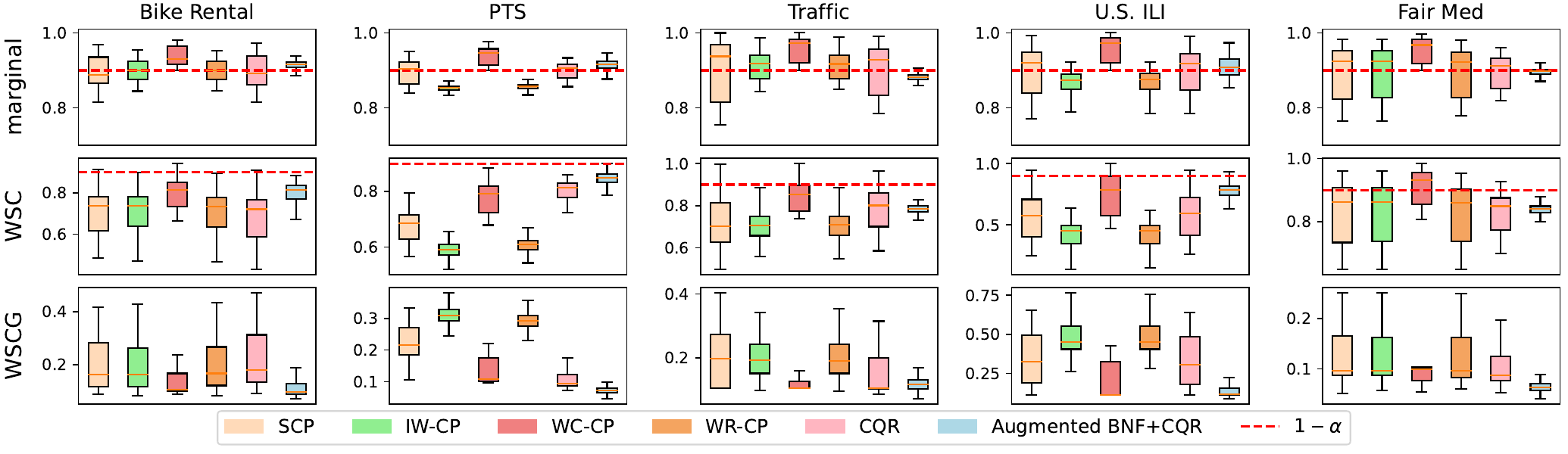}
  % \vspace{-15pt}
  \caption{Marginal coverage, WSC, and WSCG of Augmented BNF+CQR and baselines with $1-\alpha=0.9$: our method achieves the lowest WSCG and brings the marginal coverage and WSC close to the expected confidence.}
  \label{fig: aug_coverage} 
  \vspace{-13pt}
\end{figure*}
\vspace{-10pt}
\subsection{Main Result}
We evaluate the combination of Augmented BNF and CQR, along with five baseline methods, across 10 independent trials for each dataset. The results are summarized in Figure~\ref{fig: aug_coverage}, which presents box plots of coverage metrics under $1-\alpha = 0.9$. For each trial, 100 random mixtures were generated as test sets. Our approach consistently achieves marginal coverage and WSC values close to the desired confidence level $1-\alpha=0.9$, and obtains the lowest WSCG. We examine the generalization ability of the Augmented BNF across varying sample sizes and $K$ values in Appendix~\ref{appendix: approximation ability}. Besdies, we evaluate the performance of Augmented BNF across a range of $1-\alpha$ in Appendix~\ref{appendix: alpha ablation}.

\begin{wrapfigure}[8]{r}{7.1cm}
\centering
\captionsetup{singlelinecheck = false, skip=5pt, justification=centering}
\vspace{-10pt}
  \includegraphics[width=0.52\textwidth]{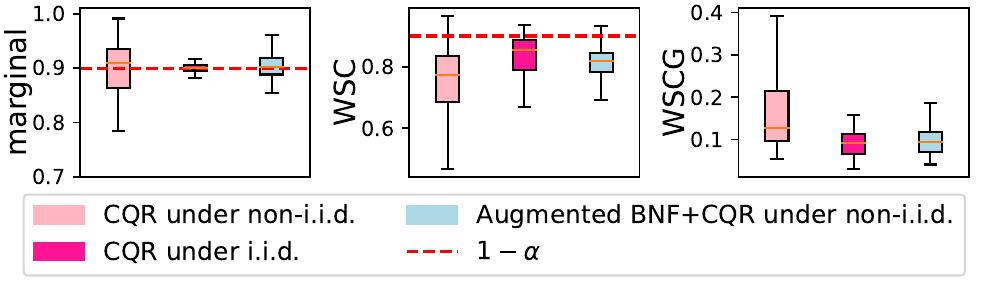}
  \vspace{-16pt}
  \caption{
  Comparison of Augmented BNF + CQR under distribution shift with CQR in shifted and i.i.d. settings.}
  \label{fig: approximation} 
\end{wrapfigure}
Since Eq.~(\ref{eq: joint Wasserstein minimization under random mixture}) is optimized empirically, the transformed test data may not perfectly follow the calibration distribution. As a result, Augmented BNF cannot fully eliminate WSCG or exactly achieve the target marginal coverage in Figure~\ref{fig: aug_coverage}. To assess this gap, we compare Augmented BNF + CQR under distribution shift with CQR in both shifted and i.i.d. settings, showing that our method closely approaches i.i.d. performance in Figure~\ref{fig: approximation}. Notably, even under i.i.d. data, CQR exhibits nonzero WSCG due to its own approximation error. Appendix~\ref{appendix: lower bound} further derives lower bounds on coverage under imperfect transformations and the approximation error.
\vspace{-3pt}
\section{Discussion}
\vspace{-3pt}
\subsection{Feature vs. Stochastic Conditioning in BNF}~\label{sec: feature conditioning}
A natural extension of the original BNF in Eq.~(\ref{eq: original BNF}) is to condition the $Y$-transformation on features, denoted by $f_{\theta_Y}^{\text{fea}}(y; x)$, in order to explicitly capture input dependence. We refer to this variant as \textbf{Feature-Conditioned BNF}. However, this modification exacerbates the curse of dimensionality, worsening the rate from $|\mathcal{S}_{D^{k}}|/(d+2)$ to $|\mathcal{S}_{D^{k}}|/(2d+1)$, where $d$ is the feature dimension. Consequently, it yields less robust conditional coverage than Augmented BNF, as shown in Figure~\ref{fig: feature aggregated}.

In contrast, $f_{\theta_Y}^\text{aug}(y; \varepsilon)$ in Augmented BNF can be viewed as conditioning on a simple one-dimensional Gaussian noise variable $\varepsilon$. Although this design leaves the learning of feature dependence entirely to the joint Wasserstein minimization in Eq.~(\ref{eq: joint Wasserstein minimization}), the injected noise $\varepsilon$ effectively introduces latent degrees of freedom, turning $f_{\theta_Y}^{\text{aug}}$ into a stochastic mapping. This allows the model to represent \emph{a distribution over transformations}, thereby substantially increasing expressiveness beyond a deterministic map.

\begin{wrapfigure}[10]{r}{7.1cm}
\centering
\captionsetup{singlelinecheck = false, skip=5pt, justification=justified}
  \includegraphics[width=0.52\textwidth]{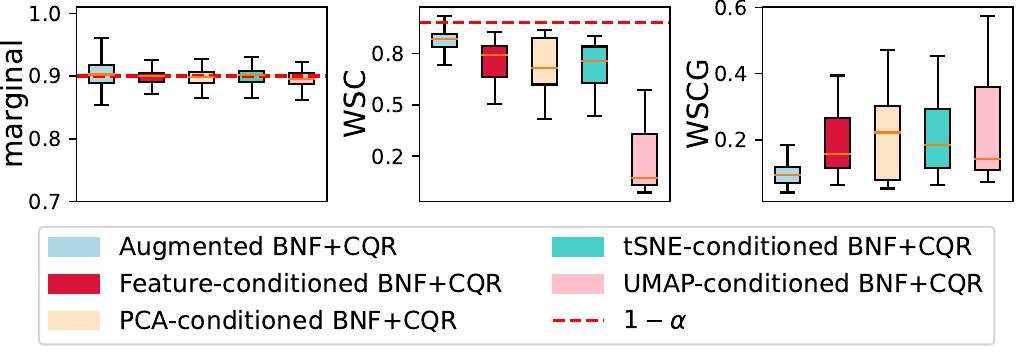}
  \vspace{-10pt}
  \caption{
  Comparison of conditioning the transformation of $Y$ on the feature and its one-dimensional projection. The result is aggregated over datasets.}
  \label{fig: feature aggregated} 
\end{wrapfigure}
Replacing the original input $x$ of $f^{\text{fea}}_{\theta_Y}(y;x)$ with a one-dimensional representation $\tilde{x}$ (e.g., via PCA~\citep{abdi2010principal}, t-SNE~\citep{van2008visualizing}, or UMAP~\citep{mcinnes2018umap}) can alleviate the curse of dimensionality. Nevertheless, as a deterministic and compressed projection, $\tilde{x}$ inevitably loses information from the original input and lacks the flexibility of stochastic conditioning. Consequently, it is less expressive than $\varepsilon$ and less informative $x$, leading to inferior performance, as shown in Figure~\ref{fig: feature aggregated}. The results for each dataset are shown in Figure~\ref{fig: feature}.
% \begin{figure*}[t]
% \centering
% \captionsetup{singlelinecheck = false, justification=justified}
%   \includegraphics[scale=0.4]{fig_feature.pdf}
%   \vspace{-15pt}
%   \caption{Comparison of conditioning the $Y$-transformation on feature and its one-dimensional projection.}
%   \label{fig: feature}
%     \vspace{-15pt}
% \end{figure*}

\subsection{Enhancing Prediction Efficiency with Source-Conditioned Transformations}
Prediction efficiency, typically measured by the size of the prediction set, is a key metric in CP, with smaller sets being more informative at a fixed coverage level. In Augmented BNF, however, the sampled noise $\varepsilon_{n+1}$ in Eq.~(\ref{eq: aug BNF prediction set}) is source-agnostic, preventing $f_{\theta_Y}^{\text{aug}}$ from identifying the origin of a test sample. To maintain valid coverage across all sources, the prediction set $\mathcal{C}_\text{BNF}^{\text{aug}}(X_{n+1})$ must therefore account for all possibilities, leading to larger sets. To improve efficiency, we introduce \textbf{Augment-Conditioned BNF}, which incorporates source-specific conditioning into the transformation. This allows the model to better distinguish among sources and produce smaller prediction sets while preserving coverage robustness, as demonstrated in Figure~\ref{fig: coverage size aggregated}. Appendix~\ref{appendix: prediction efficiency} provides a detailed introduction to Augment-Conditioned BNF and its performance on each dataset.
\begin{figure*}[h]
\centering
\captionsetup{singlelinecheck = false, justification=justified}
  \includegraphics[scale=0.4]{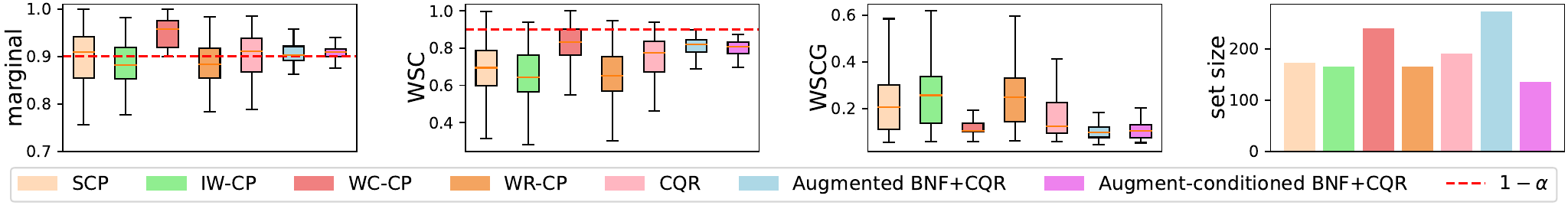}
  \vspace{-5pt}
  \caption{Standard Augmented BNF can produce large prediction sets under MSDG. The Augment-Conditioned variant significantly reduces set size and preserves coverage performance. The result is aggregated over datasets.} 
  \label{fig: coverage size aggregated}
    \vspace{-15pt}
\end{figure*}

\subsection{Distribution Shift as Label Perturbation}\label{sec: perturbation}
To further validate our method, we introduce an additional shift form via label perturbation~\cite{sesia2023adaptive, einbinder2022conformal}. Specifically, we sample perturbed labels uniformly from the interval $[Y,1.5Y]$
. This scale-based perturbation ensures that the induced shift is sufficiently pronounced. We construct ten such shifted environments and train BNF to transport them back to the unperturbed calibration distribution. During inference, we evaluate on 100 randomly shifted distributions. In Figure~\ref{fig: perturbation aggregated}, our method achieves a favorable trade-off among robust conditional coverage and high prediction efficiency, proving effectiveness beyond the multi-source setup. The results for each dataset are presented in Figure~\ref{fig: perturbation}.
\begin{figure*}[h]
\centering
\captionsetup{singlelinecheck = false, justification=justified}
  \includegraphics[scale=0.4]{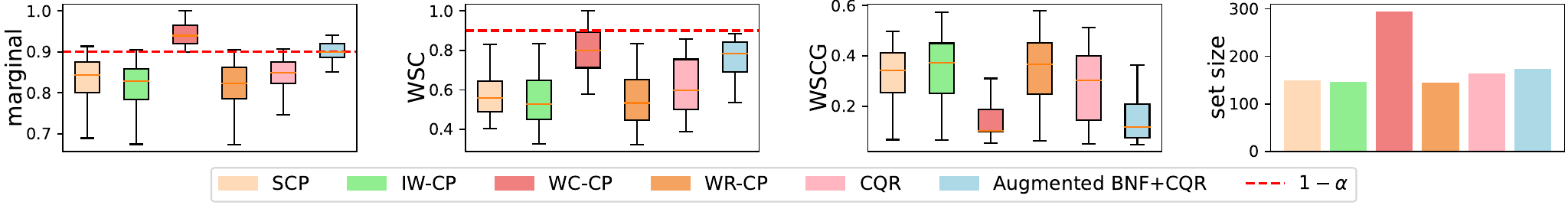}
  \vspace{-5pt}
  \caption{Distribution shift is induced by label perturbation. Our method consistently maintains comparably robust coverage without a significant cost in prediction efficiency. The result is aggregated over datasets.} 
  \label{fig: perturbation aggregated}
    \vspace{-15pt}
\end{figure*}
\vspace{-3pt}
\section{Conclusion}
\vspace{-3pt}
This work proposes the Conditional Coverage Gap (CCG) to evaluate the robustness of conditional coverage at a given test input, and defines the Integrated Coverage Gap (ICG) as its expectation over the test feature distribution. We bound ICG using the Wasserstein distance \( W(P_{XY}, Q_{XY}) \), capturing how distribution shift propagates from the data space to the conformal score space. To ensure \( 1 - \alpha \) conditional coverage under shift, we introduce the Branched Normalizing Flow (BNF). The invertibility of BNF enables mapping adaptive prediction sets from \( P_{XY} \) to \( Q_{XY} \), while the branched structure allows input $X_{n+1}$ transformation without needing \( Y_{n+1} \) at test time. BNF is applied to both synthetic and real-world distribution shifts to validate its effectiveness.

\bibliographystyle{unsrt}
\bibliography{neurips_2026}

%%%%%%%%%%%%%%%%%%%%%%%%%%%%%%%%%%%%%%%%%%%%%%%%%%%%%%%%%%%%
\newpage
\appendix

\section{Insight into scaling constants of the ICG bound}\label{appendix: scaling constants}

We provide more intuitive explanations of the scaling constants in Eq.~(\ref{eq: final bound of ICG}), clarifying their roles and implications for CP.

First, the term $L$, representing the Lebesgue density bound of $P_{V|x}$, captures the concentration of conformal scores at $x$. A higher $L$ indicates that the calibration scores are tightly clustered around certain values of $v$, which makes the conditional coverage more sensitive to distribution shifts in test conformal scores. Hence, this highlights how the shape of the calibration conformal score distribution directly influences coverage robustness.

Second, $\kappa$ provides an interpretation of how the score function $s(x,y)$ influences robustness under distribution shift. Specifically, the continuity constant $\kappa\geq(|s(x,y_1)-s(x,y_2)|)/(|y_1-y_2|)$ for all $y_1,y_2\in\mathcal{Y},x\in\mathcal{X}$. It captures the sensitivity of the score function $s(x,y)$ to changes in the label $y$, given a fixed input $x$. A smaller $\kappa$ implies that the conformal score is relatively insensitive to variations in the label, meaning that even under a large concept shift (i.e., large $W(P_{Y|x}, Q_{Y|x})$), the induced shift in conformal scores $W(P_{V|x}, Q_{V|x})$ remains small.

Lastly, the term $\eta$, introduced in Theorem~\ref{theorem: upper bound of conditional Wasserstein distance}, quantifies the extent to which the concept shift contributes to the overall joint distribution shift. A smaller $\eta$ indicates that most of the distributional difference between $P_{XY}$ 
 and $Q_{XY}$ does not stem from the difference between $P_{Y|x}$ 
 and $Q_{Y|x}$  for $x\in\mathcal{X}$. In such cases, the impact of $W(P_{XY},Q_{XY})$ on the coverage gap is limited, and accordingly, the upper bound in Eq.~(\ref{eq: final bound of ICG}) becomes tighter.

 \section{Finite-sample approximation of Wasserstein distance}\label{appendix: finite sample analysis}
In practice, the population forms of calibration and test distributions are typically inaccessible, so we may approximate $W(P_{XY},Q_{XY})$ based on empirical distributions.

\begin{definition}[$p$-Wasserstein Distance between Empirical Distributions~\citep{panaretos2019statistical}]\label{def: Wassertein distance between empirical distributions}
    Let \( \{x_i\}_{i=1}^n \sim \mu_X \) and \( \{x'_j\}_{j=1}^m \sim \nu_X \) be i.i.d.\ samples from two distributions on a metric space \( (\mathcal{X}, d_\mathcal{X}) \). The Dirac measure \( \varepsilon_{x} \) is the point mass at \( x \in \mathcal{X} \). The empirical measures are defined as
    \[
    \widehat{\mu}_X = \frac{1}{n} \sum_{i=1}^n \varepsilon_{x_i}, \quad 
    \widehat{\nu}_X = \frac{1}{m} \sum_{j=1}^m \varepsilon_{x'_i}.
    \]
    $C\in\mathbb{R}^{n\times m}$ is a cost matrix where each element $C_{ij}=d_\mathcal{X}(x_i,x'_j)$ measures the distance between sample $x_i$ from $\widehat{\mu}_X$ and $x'_j$ from $\widehat{\nu}_X$. Let \( \gamma \in \mathbb{R}^{n \times m} \) be a transportation plan matrix, where each \( \gamma_{ij} \geq 0 \) represents the mass transported from \( x_i \) to \( x'_j \). The set of admissible transport plans is
\[
\Gamma(\widehat{\mu}_X, \widehat{\nu}_X) = \left\{ \gamma \in \mathbb{R}_{\geq 0}^{n \times m} \;\middle|\; \sum_{j=1}^m \gamma_{ij} = \frac{1}{n},\; \sum_{i=1}^n \gamma_{ij} = \frac{1}{m} \right\}.
\]
    The \( p \)-Wasserstein distance between empirical distributions \( \widehat{\mu}_X \) and \( \widehat{\nu}_X \) is then given by
\[
{W}_p(\widehat{\mu}_X, \widehat{\nu}_X) = \left( \min_{\gamma \in \Gamma(\widehat{\mu}_X, \widehat{\nu}_X)} \sum_{i=1}^n \sum_{j=1}^m \gamma_{ij} \, C_{ij}^p \right)^{1/p}.
\]
\end{definition}
Let \( \widehat{P}_{XY} \) and \( \widehat{Q}_{XY} \) be the empirical distributions based on \( n \) and $m$ i.i.d.\ samples drawn from \( P_{XY} \) and \( Q_{XY} \), respectively. Our goal is to bound the deviation between the empirical and population Wasserstein distances, i.e., to analyze how $
W(\widehat{P}_{XY}, \widehat{Q}_{XY})$ converges to $W(P_{XY}, Q_{XY})$ as $n$ increases.
\begin{definition}[Upper Wasserstein Dimension~\citep{dudley1969speed}]
    Given a set $\mathcal{A}\subseteq\mathcal{X}$, the $\epsilon$-covering number, denoted $\mathcal{N}_\epsilon(\mathcal{A})$, is the smallest $n$ such that $n$ closed balls, $\mathcal{U}_1,...,\mathcal{U}_n$, of diameter $\epsilon$ achieve $\mathcal{A}\subseteq\cup_{1\leq i\leq m}\mathcal{U}_i$. For a distribution $\mu_X$ in $\mathcal{X}$, the $(\epsilon,\zeta)$-dimension is $d_\epsilon(\mu_X,\zeta)={-\log(\inf\{\mathcal{N}_\epsilon(\mathcal{A}):\mu_X(\mathcal{A})\geq1-\zeta\})}/{\log\epsilon}$. The upper Wassersteion dimension with $p=1$ is 
    \begin{equation}
        d_W(\mu_X)=\inf\{\varphi\in(2,\infty):\limsup\nolimits_{\epsilon\rightarrow0}d_\epsilon(\mu_X,\epsilon^{\frac{\varphi}{\varphi-2}})\leq \varphi\}.
    \end{equation} 
\end{definition}
\begin{theorem}\label{theorem: Wasserstein distance convergence to expectation}
    Given a probability measure $\mu_X$ in space $\mathcal{X}$, let $\sigma>d_W(\mu_X)$. If $\widehat{\mu}_X$ is an empirical measure corresponding to $n$ i.i.d. samples from $\mu_X$, $\exists\lambda\in\mathbb{R}$ such that $\mathbb{E}[W(\mu_X,\widehat{\mu}_X)]\leq\lambda n^{-1/\sigma}$. Furthermore, for $t>0$, $\textup{Pr}(W(\mu_X,\widehat{\mu}_X)\geq\mathbb{E}[W(\mu_X,\widehat{\mu}_X)]+t)\leq e^{-2nt^2}$~\citep{weed2019sharp}. 
\end{theorem}
\begin{theorem}\label{theorem: Wasserstein approximation}
    Given probability measures $\mu_X$ and $\nu_X$ in space $\mathcal{X}$, let $\sigma_\mu>d_W(\mu_X)$ and $\sigma_\nu>d_W(\nu_X)$. Denote $\widehat{\mu}_X$ and $\widehat{\nu}_X$ empirical measures corresponding to $n$ and $m$ i.i.d. samples from $\mu_X$ and $\nu_X$, respectively. For $t_\mu,t_\nu>0$, $\exists\lambda_\mu, \lambda_\nu>0$ with probability at least $(1-e^{-2n{t_\mu}^2})(1-e^{-2m{t_\nu}^2})$ that
    \begin{equation}\label{eq: Wasserstein approximation with absolute}
        \left|W(\mu_{X}, \nu_{X}) - W(\widehat{\mu}_{X}, \widehat{\nu}_{X})\right| \leq \lambda_\mu n^{-1/\sigma_\mu} + \lambda_\nu m^{-1/\sigma_\nu} + t_\mu + t_\nu.
    \end{equation}
\end{theorem}
A related theorem is proposed in~\citep{xu2025wassersteinregularized}, though without accounting for the signs of $\lambda_\mu$ and $\lambda_\nu$. Based on Theorem~\ref{theorem: Wasserstein approximation}, if $\sigma _P>d _W(P _{XY})$ and $\sigma _Q>d _W(Q _{XY})$, for $t _P, t _Q > 0$, there are $\lambda _P, \lambda _Q>0$ with a probability at least  $(1-e^{-2nt _P^2})(1-e^{-2mt _Q^2})$ that
\begin{equation}\label{eq: Wasserstein approximation with absolute on P and Q}
        \left|W(P _{XY}, Q _{XY}) - W(\widehat{P} _{XY}, \widehat{Q} _{XY})\right| \leq \lambda _P n^{-1/\sigma _P} + \lambda _Q m^{-1/\sigma _Q} + t_P + t_Q.
\end{equation}
As \( n \) and $m$ increase, the bound in Eq.~(\ref{eq: Wasserstein approximation with absolute on P and Q}) decreases, thereby improving the approximation of the empirical Wasserstein distance. At the same time, the probability \( (1 - e^{-2n t_P^2})(1 - e^{-2m t_Q^2}) \) increases, indicating that the bound holds with higher confidence.

\section{Additional theoretical statements}\label{appendix: additional theoretical statements}
\subsection{Supporting propositions}
\begin{proposition}\label{proposition: equivalence with monoticity} Let $f: \mathcal{X} \to \mathcal{Y}$ be an invertible univariate function, where $\mathcal{X}, \mathcal{Y} \subseteq \mathbb{R}$. Let $C = [y_{\textnormal{lo}}, y_{\textnormal{hi}}] \subseteq \mathcal{Y}$ be a closed interval. Then for any $y \in \mathcal{Y}$, the following equivalence holds:
\[
y \in C \quad \Longleftrightarrow \quad f^{-1}(y) \in \{x \in \mathcal{X} : f(x) \in C\}.
\]
\end{proposition}

\begin{proof}
Since $f$ is an invertible univariate function, it must be strictly monotonic, either strictly increasing or decreasing.

\textbf{Case 1:} Suppose $f$ is strictly increasing. Then $f^{-1}$ is also strictly increasing.

\begin{itemize}
\item[$\Rightarrow$] If $y \in C = [y_{\textnormal{lo}}, y_{\textnormal{hi}}]$, then by monotonicity,
\[
f^{-1}(y_{\textnormal{lo}}) \leq f^{-1}(y) \leq f^{-1}(y_{\textnormal{hi}}),
\]
so \( f^{-1}(y) \in [f^{-1}(y_{\textnormal{lo}}), f^{-1}(y_{\textnormal{hi}})] \). Since \( f \) is strictly increasing, this implies
\[
[f^{-1}(y_{\textnormal{lo}}), f^{-1}(y_{\textnormal{hi}})]=\{x \in \mathcal{X} : f(x) \in C\}
\]
and thus \( f^{-1}(y) \in \{x \in \mathcal{X} : f(x) \in C\} \).

\item[$\Leftarrow$] If \( f^{-1}(y) \in \{x \in \mathcal{X} : f(x) \in C\} \), then equivalently we can derive $y \in C$.
\end{itemize}

\textbf{Case 2:} Suppose $f$ is strictly decreasing. Then $f^{-1}$ is also strictly decreasing.

\begin{itemize}
\item[$\Rightarrow$] If $y \in C = [y_{\textnormal{lo}}, y_{\textnormal{hi}}]$, then
\[
f^{-1}(y_{\textnormal{lo}}) \geq f^{-1}(y) \geq f^{-1}(y_{\textnormal{hi}}),
\]
so \( f^{-1}(y) \in [f^{-1}(y_{\textnormal{hi}}), f^{-1}(y_{\textnormal{lo}})] \). Again, since \( f \) is decreasing,
\[
[f^{-1}(y_{\textnormal{hi}}), f^{-1}(y_{\textnormal{lo}})] = \{x \in \mathcal{X} : f(x) \in C\},
\]
which implies \( f^{-1}(y) \in \{x \in \mathcal{X} : f(x) \in C\} \).

\item[$\Leftarrow$] If \( f^{-1}(y) \in \{x \in \mathcal{X} : f(x) \in C\} \), then again \( y \in C \).
\end{itemize}

In either case, the equivalence holds.
\end{proof}

\begin{proposition}\label{proposition: equivalence without monoticity}
Let $f: \mathcal{X} \to \mathcal{Y}$ be a univariate function, where $\mathcal{X}, \mathcal{Y} \subseteq \mathbb{R}$. Let $C \subseteq \mathcal{Y}$ be a closed interval. Then for $a \in \mathcal{X}$, it holds that:
\begin{equation*}
    f(a)\in C \iff a \in \{x \in \mathcal{X} : f(x) \in C\}.
\end{equation*}
\end{proposition}

\begin{proof}
The statement is a direct consequence of the definition of the set $\{x \in \mathcal{X} : f(x) \in C\}$. By definition, $a$ belongs to this set if and only if $a \in \mathcal{X}$ and $f(a) \in C$. Since $a \in \mathcal{X}$ is already assumed, the condition reduces to:
$f(a) \in C \iff a \in \{x \in \mathcal{X} : f(x) \in C\}$.
\end{proof}
We visualize Proposition~\ref{proposition: equivalence with monoticity} and Proposition~\ref{proposition: equivalence without monoticity} in Figure~\ref{fig: proposition}.

\begin{figure*}[h]
\centering
\captionsetup{singlelinecheck = false, justification=centering}

  \includegraphics[scale=0.4]{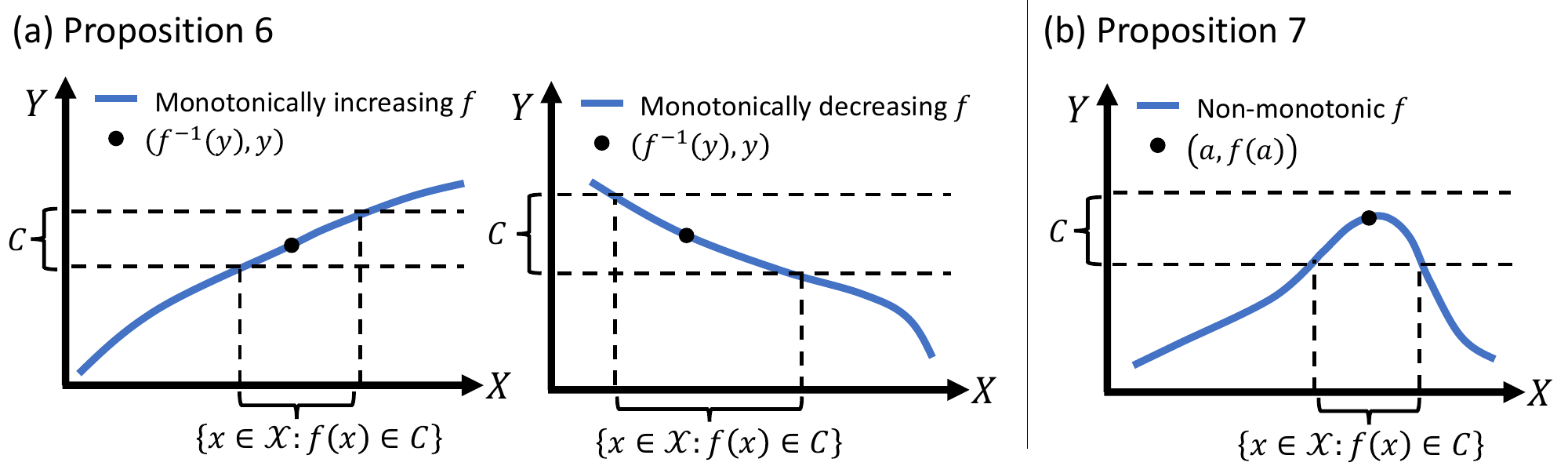}
  % \vspace{-10pt}
  \caption{Characterization of preimage membership under \textbf{(a)} monotonic and \textbf{(b)} non-monotonic functions.}
  \label{fig: proposition} 
  % \vspace{-20pt}
\end{figure*}

\subsection{Proof of Theorem~\ref{theorem: Wasserstein distance with Lipschitz continuity}}
\begin{proof}
Let $\mu_{XY}$ and $\nu_{XY}$ be probability measures on the metric space $(\mathcal{X}\times\mathcal{Y}, d_\mathcal{XY})$, where $d_{\mathcal{X}\mathcal{Y}}((x_1,y_1),(x_2,y_2)):=||(d_\mathcal{X}(x_1,x_2),d_\mathcal{Y}(y_1,y_2))||_2$. Let $s: \mathcal{X}\times\mathcal{Y} \to \mathcal{V}$ be a measurable function such that $s(x,y)=v$. In metric space $(\mathcal{V}, d_\mathcal{V})$, denote $\mu_V$ the probability measure of $s(X,Y)$ for $(X,Y) \sim \mu_{XY}$. Also, let $\nu_V$ be the probability measure of $s(X,Y)$ for $(X,Y) \sim \nu_{XY}$. Denote $\Gamma_{V|x}=\Gamma(\mu_{V|x},\nu_{V|x})$ and $\Gamma_{Y|x}=\Gamma(\mu_{Y|x},\nu_{Y|x})$. By Theorem 1 in~\citep{xu2025wassersteinregularized}, we derive
\begin{equation}\label{eq: W distance with pushforward}
    \begin{split}
    &W(\mu_{V|x},\nu_{V|x})=\inf_{\gamma\in\Gamma_{V|x}}\int_{\mathcal{V}\times\mathcal{V}}{d_\mathcal{V}(v_1,v_2)\dd{\gamma(v_1,v_2)}}\\&=\inf_{\gamma\in\Gamma_{Y|x}}\int_{\mathcal{Y}\times\mathcal{Y}}{d_\mathcal{V}(s(x,y_1),s(x,y_2))\dd{\gamma(y_1,y_2)}}.
    \end{split}
\end{equation}
Consider $\gamma^*\in\Gamma_{Y|x}$ is the optimal transport plan for $W(\mu_{Y|x},\nu_{Y|x})$. However, $\gamma^*$ is not necessarily optimal for obtaining $W(\mu_{V|x},\nu_{V|x})$ in Eq.~(\ref{eq: W distance with pushforward}), so we have 
\begin{equation}\label{eq: inequality by gamma^*}
        W(\mu_{V|x},\nu_{V|x})\leq\int_{\mathcal{Y}\times\mathcal{Y}}{d_\mathcal{V}(s(x,y_1)-s(x,y_2))\dd{\gamma^*(y_1,y_2)}}.
\end{equation}
Given that the function $s$ is continuous with constant $\kappa$ conditioned on $x$, we have $\frac{d_\mathcal{V}(s(x,y_1),s(x,y_2))}{d_\mathcal{Y}(y_1-y_2)}\leq \kappa$, $\forall x\in \mathcal{X},y_1,y_2\in\mathcal{Y}$, so the following inequality holds that
\begin{equation}\label{eq: kappa inequality in proof}
    \int_{\mathcal{Y}\times\mathcal{Y}}{d_\mathcal{V}(s(x,y_1),s(x,y_2))\dd{\gamma^*(y_1,y_2)}}\leq\int_{\mathcal{Y}\times\mathcal{Y}}{\kappa\cdot d_\mathcal{Y}(y_1,y_2)\dd{\gamma^*(y_1,y_2)}}=\kappa\cdot W(\mu_{Y|x},\nu_{Y|x}).
\end{equation}
Finally, combining Eq.~(\ref{eq: inequality by gamma^*}) and Eq.~(\ref{eq: kappa inequality in proof}), we can conclude that
\begin{equation}
    W(\mu_{V|x},\nu_{V|x})\leq\kappa\cdot W(\mu_{Y|x},\nu_{Y|x}).
\end{equation}
\end{proof}
\subsection{Proof of Theorem~\ref{theorem: upper bound of conditional Wasserstein distance}}
\begin{proof}
Let $\mu_{XY}$ and $\nu_{XY}$ be probability measures on the metric space $(\mathcal{X}\times\mathcal{Y}, d_{\mathcal{X}\mathcal{Y}})$, where $d_{\mathcal{X}\mathcal{Y}}((x_1,y_1),(x_2,y_2)):=||(d_\mathcal{X}(x_1,x_2),d_\mathcal{Y}(y_1,y_2))||_2$. A joint distribution shift results in $\mu_{X}\neq\nu_{X}$, $\mu_{Y|X}\neq\nu_{Y|X}$. 

For any $\gamma_{XYXY}\in\Gamma(\mu_{XY},\nu_{XY})$, denote $\gamma_{XX}(x_1,x_2)=\int_{\mathcal{Y}^2}\dd\gamma_{XYXY}(x_1,y_1,x_2,y_2)$. Thereby, we can derive
\begin{equation}\label{eq: proof 2 eq 1}
    \begin{split}
        &\int_{\mathcal{X}^2\times\mathcal{Y}^2}d_{\mathcal{XY}}((x_1,y_1),(x_2,y_2))\dd\gamma_{XYXY}(x_1,y_1,x_2,y_2)\\&\geq\int_{\mathcal{X}^2\times\mathcal{Y}^2}d_\mathcal{Y}(y_1,y_2)\dd\gamma_{XYXY}(x_1,y_1,x_2,y_2)\\&\geq\int_{\mathcal{X}^2\times\mathcal{Y}^2}d_\mathcal{Y}(y_1,y_2)I(x_1=x_2)\dd\gamma_{XYXY}(x_1,y_1,x_2,y_2)\\&=\int_{\mathcal{X}^2}\left(\int_{\mathcal{Y}^2}d_\mathcal{Y}(y_1,y_2)\dd\gamma_{YY|x_1x_2}(y_1,y_2)\right)I(x_1=x_2)\dd\gamma_{XX}(x_1,x_2)\\&=\int_{\mathcal{X}^2}\left(\int_{\mathcal{Y}^2}d_\mathcal{Y}(y_1,y_2)\dd\gamma_{YY|x_1x_1}(y_1,y_2)\right)\dd\gamma_{XX}(x_1,x_1).
    \end{split}
\end{equation}
Consider $\gamma^*_{XYXY}\in\Gamma(\mu_{XY},\nu_{XY})$ that satisfies
\begin{equation}\label{eq: proof 2 eq 2}
W(\mu_{XY},\nu_{XY})=\int_{\mathcal{X}^2\times\mathcal{Y}^2}d_{\mathcal{XY}}((x_1,y_1),(x_2,y_2))\dd\gamma^*_{XYXY}(x_1,y_1,x_2,y_2).
\end{equation}
However, $\gamma^*_{YY|x_1x_1}$ is not necessarily the optimal transport plan of $W(\mu_{Y|x_1},\nu_{Y|x_1}), \forall x_1\in\mathcal{X}$, so
\begin{equation}\label{eq: proof 2 eq 3}
W(\mu_{Y|x_1},\nu_{Y|x_1})\leq\int_{\mathcal{Y}^2}d_\mathcal{Y}(y_1,y_2)\dd\gamma^*_{YY|x_1x_1}(y_1,y_2).
\end{equation}
Therefore, after plugging Eq.~(\ref{eq: proof 2 eq 2}) and Eq.~(\ref{eq: proof 2 eq 3}) into Eq.~(\ref{eq: proof 2 eq 1}) and simplifying $x_1$ as $x$, we obtain
\begin{equation}
W(\mu_{XY},\nu_{XY})\geq\int_{\mathcal{X}^2}W(\mu_{Y|x},\nu_{Y|x})\dd\gamma^*_{XX}(x,x).
\end{equation}
Given $\eta>0$ that satisfies
\begin{equation}\label{eq: eta requirement in proof}
\eta\int_{\mathcal{X}^2}W(\mu_{Y|x},\nu_{Y|x})\dd\gamma^*_{XX}(x,x)\geq\int_{\mathcal{X}}W(\mu_{Y|x},\nu_{Y|x})\dd\nu_X(x),
\end{equation}
we can consequently prove
\begin{equation}         
\eta\cdot W(\mu_{XY},\nu_{XY})\geq\int_{\mathcal{X}}W(\mu_{Y|x},\nu_{Y|x})\dd\nu_X(x).
\end{equation}
\end{proof}
We would like to further justify the necessity of introducing $\eta$ to satisfy Eq.~(\ref{eq: eta requirement in proof}).  

Considering $\int _\mathcal{X}^2 \dd\gamma^* _{XX}(x,x)=\int _\mathcal{X}^2 I(x _1=x _2) \dd\gamma^* _{XX}(x _1,x _2)$, we denote 
\begin{equation}
    \psi(\mathcal{A})=\gamma^* _{XX}\left(\{(x _1,x _2)\in\mathcal{X}^2:x _1=x _2\in\mathcal{A}\}\right)=\gamma^* _{XX}(\mathcal{A}\times\mathcal{A})
    \text{, }\forall  \mathcal{A}\subset\mathcal{X}.
\end{equation} 

As $\nu _X$ is a projection of $\gamma^* _{XX}$, we have $\nu _X(\mathcal{A})=\gamma^* _{XX}(\mathcal{A}\times\mathcal{X})\geq\psi(\mathcal{A})$. By the Radon-Nikodym theorem~\citep{fonseca2007modern}, there exists a density $\rho(x)\geq0$ such that 
\begin{equation}
    \psi(\mathcal{A})=\int _\mathcal{A}\rho(x)\dd\nu _X(x).
\end{equation} 
Since $\psi(\mathcal{A})\leq\nu(\mathcal{A})$, we can derive $\int _\mathcal{A}\rho(x)\dd\nu _X(x)\leq\int _\mathcal{A}1\dd\nu _X(x)$ for all $\mathcal{A}$. This forces $\rho(x)\leq1$ almost everywhere on $\nu _X$. 
As a result, we conclude that
\begin{equation}\label{eq: constraint in proof 2}
    \int _{\mathcal{X}^2} W(\mu _{Y|x},\nu _{Y|x})\dd\psi(x)=\int _\mathcal{X} W(\mu _{Y|x},\nu _{Y|x})\rho(x)\dd\nu _X(x)\leq\int _\mathcal{X} W(\mu _{Y|x},\nu _{Y|x})\dd\nu_X(x).
\end{equation} Therefore, we introduce a constant $\eta$ to reverse the inequality in Eq.~(\ref{eq: constraint in proof 2}).
\subsection{Proof of Theorem~\ref{theorem: Wasserstein upper bound under random mixture}}
\begin{proof}
    For each $k\in\{1,...,K\}$, denote $\gamma^k\in\Gamma(\mu_{XY},\nu_{XY}^k)$ the optimal transport plan realizing $W(\mu_{XY},\nu_{XY}^k)$ such that
    \begin{equation}
        W(\mu_{XY},\nu_{XY}^k)=\int_{\mathcal{X}\times\mathcal{Y}}d_{\mathcal{X}\mathcal{Y}}(x,y)\dd\gamma^k(x,y).
    \end{equation}
    Given $\nu_{XY}=\sum_{k=1}^K\lambda_k\nu_{XY}^k$, let $\gamma^*=\sum_{k=1}^K\lambda_k\gamma^k$. Since the first marginal of $\gamma^*$ is $\mu_{XY}$ and the second marginal of $\gamma^*$ is $\sum_{k=1}^K\lambda_k\nu_{XY}^k$, it follows that $\gamma^*\in\Gamma(\mu_{XY},\nu_{XY})$. However, $\gamma^*$ is not necessarily optimal transport plan for $W(\mu_{XY},\nu_{XY})$, we conclude that
    \begin{equation}
        \begin{split}
         &W(\mu_{XY},\nu_{XY})=\inf_{\gamma\in\Gamma(\mu_{XY},\nu_{XY})}\int_{\mathcal{X}\times\mathcal{Y}}d_{\mathcal{X}\mathcal{Y}}(x,y)\dd\gamma(x,y)\leq\int_{\mathcal{X}\times\mathcal{Y}}d_{\mathcal{X}\mathcal{Y}}(x,y)\dd\gamma^*(x,y)\\&=\sum\nolimits_{k=1}^K\lambda_k\int_{\mathcal{X}\times\mathcal{Y}}d_{\mathcal{X}\mathcal{Y}}(x,y)\dd\gamma^k(x,y)=\sum\nolimits_{k=1}^K\lambda_kW(\mu_{XY},\nu_{XY}^k).
        \end{split}
    \end{equation}
\end{proof}
\subsection{Proof of Theorem~\ref{theorem: Wasserstein approximation}}

\begin{proof}
Since the Wasserstein distance satisfies the triangle inequality, the distance $W(\mu_X, \nu_X)$ can be related to the empirical distributions $ \widehat{\mu}_X$ and $\widehat{\nu}_X$ as follows:
\begin{equation}\label{eq: Wasserstein triangle ineqality}
     W(\mu_X,\nu_X)\leq W(\widehat{\mu}_X,\mu_X) + W(\widehat{\mu}_X,\nu_X) \leq W(\widehat{\mu}_X,\mu_X) +  W(\widehat{\mu}_X,\widehat{\nu}_X) + W(\widehat{\nu}_X,\nu_X).
\end{equation}
Given $\mathbb{E}[W(\mu,\widehat{\mu}_X)]\leq\lambda_\mu n^{-1/\sigma_\mu}$ and $\mathbb{E}[W(\nu_X,\widehat{\nu}_X)]\leq\lambda_\nu m^{-1/\sigma_\nu}$ from Theorem~\ref{theorem: Wasserstein distance convergence to expectation}, with probabilities at least $1-e^{-2n{t_\mu}^2}$ and $1-e^{-2m{t_\nu}^2}$, respectively, we have
\begin{equation}\label{eq: Wasserstein ineqality mu and nu}
\begin{split}
    &W(\mu_X,\widehat{\mu}_X)\leq\lambda_\mu n^{-1/\sigma_\mu}+t_\mu;\\&W(\nu_X,\widehat{\nu}_X)\leq\lambda_\nu m^{-1/\sigma_\nu}+t_\nu.
\end{split}
\end{equation}
It is reasonable to assume the two events in Eq.~(\ref{eq: Wasserstein ineqality mu and nu}) are independent, so we can apply them to Eq.~(\ref{eq: Wasserstein triangle ineqality}), and thus obtain
\begin{equation}\label{eq: Wasserstein approximation without absolute}
            W(\mu_{X}, \nu_{X}) - W(\widehat{\mu}_{X}, \widehat{\nu}_{X})\leq \lambda_\mu n^{-1/\sigma_\mu} + \lambda_\nu m^{-1/\sigma_\nu} + t_\mu + t_\nu
\end{equation}
with probability at least $(1-e^{-2n{t_\mu}^2})(1-e^{-2m{t_\nu}^2})$.

Since $\mathbb{E}[W(\mu,\widehat{\mu}_X)]$ and $\mathbb{E}[W(\nu_X,\widehat{\nu}_X)]$ are non-negative, it follows that $\lambda_\mu,\lambda_\nu\geq0$. Given that $t_\mu$ and $t_\nu$ are also positive, the right-hand side of Eq.~(\ref{eq: Wasserstein approximation without absolute}) is non-negative. Therefore, we can take the absolute value on both sides of Eq.~(\ref{eq: Wasserstein approximation without absolute}) without changing the direction of the inequality, leading to Eq.~(\ref{eq: Wasserstein approximation with absolute}).
\end{proof}

\section{Demonstration of implicit dependency}\label{appendix: demonstration of implicit dependency}
We demonstrate that $f_{\theta_Y}^{-1}(\widebar{Y}_{n+1})$ implicitly depends on $X_{n+1}$ through the composition  $\phi\circ f_{\theta_X}$, where $\phi:\mathcal{X}\rightarrow\mathcal{Y}$ is the ground truth mapping function under the calibration distribution $P_{XY}$. 
Consider a BNF $f_\theta$ is optimized by Wasserstein distance minimization in Eq.~(\ref{eq: joint Wasserstein minimization}) such that $f_{\theta\#} Q_{XY}=P_{XY}$. Therefore, for a test sample $(X_{n+1},Y_{n+1})=(x,y)\sim Q_{XY}$, it holds that
\begin{equation*}
    (\widebar{X}_{n+1},\widebar{Y}_{n+1})=f_\theta(x,y)=(f_{\theta_X}(x),f_{\theta_Y}(y))=(\widebar{x},\widebar{y})\sim P_{XY}.
\end{equation*}
As a result, $\mathcal{C}_\text{C}(\widebar{X}_{n+1})$ satisfies the conditional coverage guarantee under $P_{XY}$. Moreover, since $\bar{y} = \phi(\widebar{x})$, we obtain
\begin{equation*}
    f_{\theta_Y}^{-1}(\widebar{y})=f_{\theta_Y}^{-1}(\phi(\widebar{x}))=f_{\theta_Y}^{-1}(\phi(f_{\theta_X}(x))),
\end{equation*}
which shows that the inverse transformation $f_{\theta_Y}^{-1}(\widebar{y})$ used to construct $\mathcal{C}_\text{BNF}(X_{n+1})$ inherently captures the dependency on $X_{n+1} = x$.

We present an example to illustrate the dependency. Denote $\mathcal{U}$ and $\mathcal{N}$ uniform and Gaussian distributions, respectively. To introduce a distribution shift between $P_{XY}$ and $Q_{XY}$, let 
\begin{equation*}
\begin{split}
    &P_X=\mathcal{U}(0,1),\smallskip P_{Y|X}=\mathcal{N}(-0.5X,-0.3X^2 +0.3X);\\&
    Q_X=\mathcal{U}(0,0.8),\smallskip Q_{Y|X}=\mathcal{N}(0.25X,-0.24X^2+0.24X).
\end{split}
\end{equation*}
Figure~\ref{fig: Toy} shows how the inverse transformation $f_{\theta_Y}^{-1}$ preserve the conditional guarantee from $\mathcal{C}_\text{C}(\widebar{X}_{n+1})$ to $\mathcal{C}_\text{BNF}(X_{n+1})$ through the implicit dependency on $X_{n+1}=x$.
\begin{figure}[h]
\centering
\captionsetup{singlelinecheck = false, justification=justified}
  \includegraphics[scale=0.55]{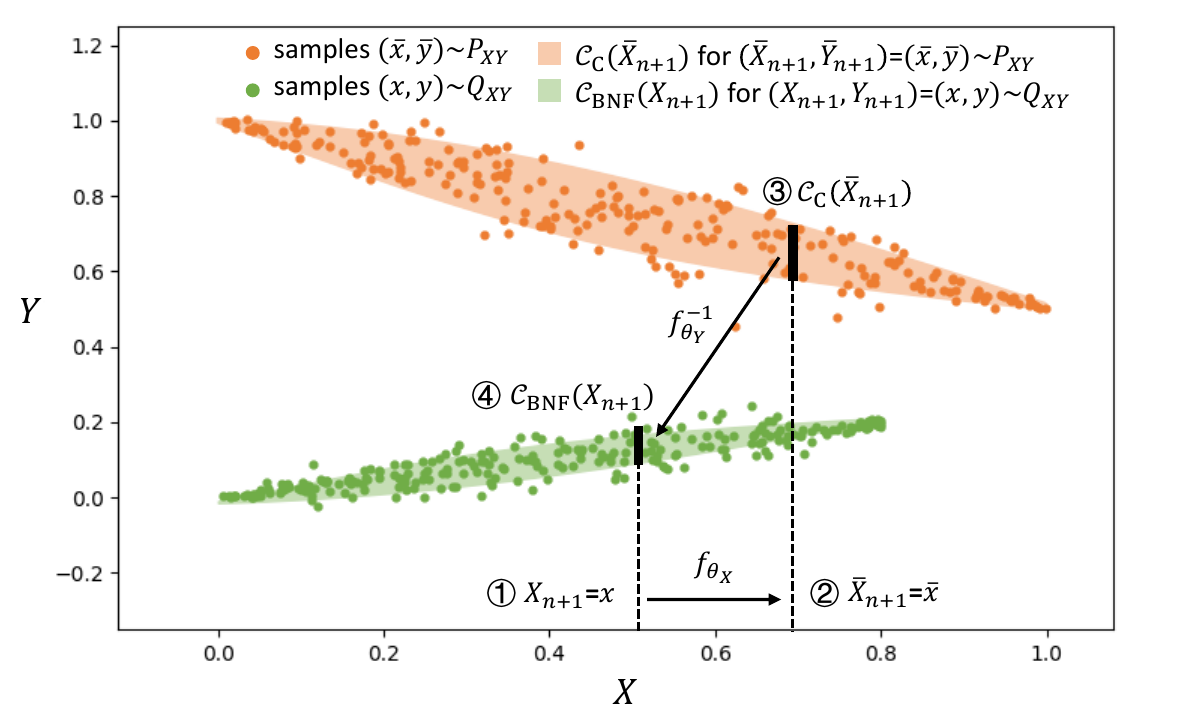}
  % \vspace{-20pt}
  \caption{Preserving conditional coverage via implicit dependency on test input. The circled numbers indicate the sequential steps to obtain the corresponding values or prediction sets.}
  \label{fig: Toy} 
  \vspace{-5pt}
\end{figure}

We admit the factorized architecture of BNF relies on a mild structural assumption of $P_{XY}$ and $Q_{XY}$ to realize the exact alignment in Eq.~(\ref{eq: exact alignment}).

Let labeling processes be $Y=\psi(X)+\psi_\text{rand}$ on $Q_{XY}$ and $Y=\phi(X)+\phi_\text{rand}$ on $P_{XY}$, where $\psi_\text{rand}$ and $\phi_\text{rand}$ are stochastic parts. Eq.~(\ref{eq: exact alignment}) implies 
\begin{equation}\label{eq: star}
    f_{\theta_Y}(\psi(X)+\psi_\text{rand})=\phi(f_{\theta_X}(X))+\phi_\text{rand}.
\end{equation}
We implement $f_{\theta_Y}$ as an augmented Real NVP that includes tanh activations over $Y$. Hence, $f_{\theta_Y}$ is not affine with respect to $\psi(X)+\psi_\text{rand}$, so in general
\begin{equation}
    f_{\theta_Y}(\psi(X)+\psi_\text{rand})\neq f_{\theta_Y}(\psi(X))+f_{\theta_Y}(\psi_\text{rand}).
\end{equation}
 Consequently, we can not split Eq.~\ref{eq: star} by $f_{\theta_Y}(\psi(X))=\phi(f_{\theta_X}(X))$ and $f_{\theta_Y}(\psi_\text{rand})=\phi_\text{rand}$.
 Accordingly, the deterministic-stochastic decomposition of Eq.~\ref{eq: star} can be naturally expressed via expectation:
 \begin{equation}
     \phi(f_{\theta_X})=\mathbb{E}_{\psi_\text{rand}}[ f_{\theta_Y}(\psi(X)+\psi_\text{rand})],
 \end{equation}
 \begin{equation}
     \phi_\text{rand} = f_{\theta_Y}(\psi(X)+\psi_\text{rand})-\mathbb{E}_{\psi_\text{rand}}[ f_{\theta_Y}(\psi(X)+\psi_\text{rand})].
 \end{equation}
 This leads to a structural assumption: for exact alignment to hold, both $\phi$ and $\phi_\text{rand}$ should be generally feature-dependent, unless $\psi_\text{rand}$
 is independent noise and $\psi$
 is constant with respect to $X$.

 Experiment result in~\ref{fig: aug_coverage} shows that the augmented factorized architecture is sufficiently expressive to capture practical distribution shifts, where both deterministic and stochastic labeling parts on $P$
 and $Q$ are typically feature-dependent, satisfying the assumption above. 

A possible way to eliminate this assumption is to condition the transformation of $Y$ on the feature $X$. However, as discussed in Section~\ref{sec: feature conditioning}, this approach introduces practical challenges, most notably exacerbating the curse of dimensionality.

\section{Comparison between normalizing flow techniques}\label{appendix: comparison between flows}
The monotonicity of the univariate $f_{\theta_Y}$ allows us to take advantage of Proposition~\ref{proposition: equivalence with monoticity} to inversely transform $\mathcal{C}_\text{A}(\widebar{X}_{n+1})$ via Eq.~(\ref{eq: equivalence with monoticity}). However, the monotonicity also limits the flexibility of $f_{\theta_Y}$, restricting the class of distributions it can model. Here, we briefly introduce several normalizing flow techniques designed for one-dimensional transformations that often struggle to map complex distributions effectively, thereby motivating the design of Augmented BNF in Section~\ref{sec: augmented BNF}. For a more comprehensive overview of normalizing flows, we refer to the survey by~\cite{kobyzev2020normalizing}.

We begin with planar flow, a fundamental transformation that expands or contracts the input space along specific directions~\citep{rezende2016variationalinferencenormalizingflows}. A planar flow is achieved by applying a linear transformation followed by a nonlinear activation, which dictates how the data is warped. To enhance expressiveness, normalizing flows are typically constructed as compositions of multiple sub-flows. We implement a BNF where each branch applies a sequence of 16 planar flows. LeakyReLU is used as the nonlinear activation function to preserve invertibility throughout the transformation.

Residual flow is built using residual connections~\citep{he2015deepresiduallearningimage}. The output of a residual connection is the sum of the original input and a transformation generated by a neural network. For these residual connections to be invertible, the transformation must have a Lipschitz constant less than 1, ensuring that the transformation does not distort the data too much. We also construct a BNF where each branch consists of 16 residual connections. The neural network within each residual connection has an architecture consisting of an input layer, two hidden layers with 128 units each, and an output layer matching the input dimension.

Both planar flow and residual flow are capable of transforming one-dimensional data. In addition, autoregressive flow~\citep{kingma2016improved, papamakarios2021normalizing} offers an alternative approach by modeling each transformation step as conditioned on the preceding ones, meaning the transformation of each sample value explicitly depends on the values that came before it. This sequential dependency enables more flexible and expressive density estimation, particularly in one-dimensional settings. However, because BNF requires deterministic transformations that are independent of input ordering, autoregressive flow is not suitable for our approach.

We illustrate the performance of BNFs constructed using planar flow and residual flow in Figure~\ref{fig: flow_ablation} and compare them against the Augmented BNF, which is implemented using a standard coupling normalizing flow, Real NVP~\citep{dinh2016density}. Detailed specifications for the Augmented BNF are provided in Appendix~\ref{appendix: architecture}. The results show that BNFs using univariate $f_{\theta_Y}$ struggle to transform complex distributions effectively, resulting in higher WSCG compared to the Augmented BNF.

\begin{figure*}[h]
\centering
\captionsetup{singlelinecheck = false, justification=justified}
  \includegraphics[scale=0.4]{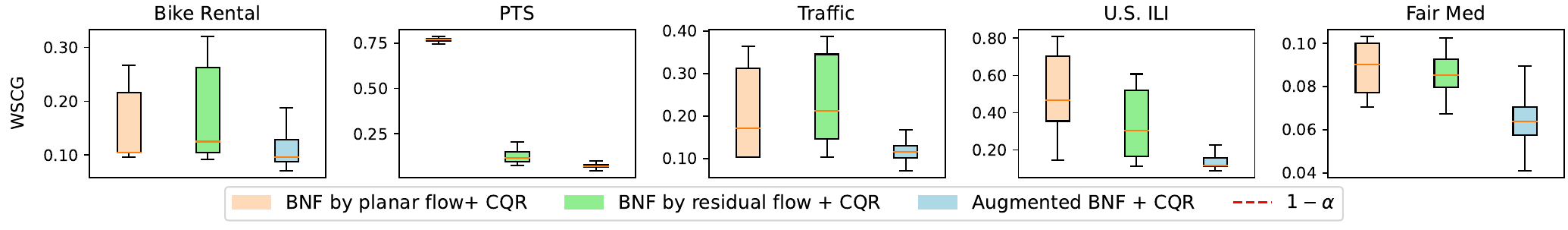}
  \caption{WSCG of BNFs constructed with planar and residual flows, compared with Augmented BNF at confidence level $1-\alpha=0.9$.}
  \label{fig: flow_ablation} 
\end{figure*}

\section{Structure of Augmented BNF via coupling flows}\label{appendix: architecture}
Both branches of Augmented BNF operate on multi-dimensional data, enabling the use of coupling flows—a technique for modeling complex high-dimensional distributions. A coupling flow usually consists of multiple coupling layers. In a coupling layer, the input is partitioned into two parts. One part remains unchanged during the transformation, while the other is modified using a neural network $c$, whose parameter $\Theta$ depends on the unchanged part. This setup ensures invertibility and allows for flexible, learnable transformations. Afterward, a permutation step is applied for higher expressiveness. In our implementation, each branch of Augmented BNF consists of a sequence of 48 coupling layers based on Real NVP~\citep{dinh2016density}, allowing the entire input to be progressively transformed. The neural network $c$ within each coupling layer follows a symmetric architecture with hidden layers of sizes 64, 128, 256, 128, and 64, mapping from the input dimension to the output dimension. Figure~\ref{fig: architecture} illustrates the structure of a coupling layer, using a random variable $Z \in \mathbb{R}^d$ with a realization $z$, and shows how both branches are constructed by stacking multiple coupling layers. The normalized Gaussian noise $\widebar{\varepsilon}$ is discarded after the transformation. 
\begin{figure*}[h]
\centering
\captionsetup{singlelinecheck = false, justification=centering}
  \includegraphics[scale=0.4]{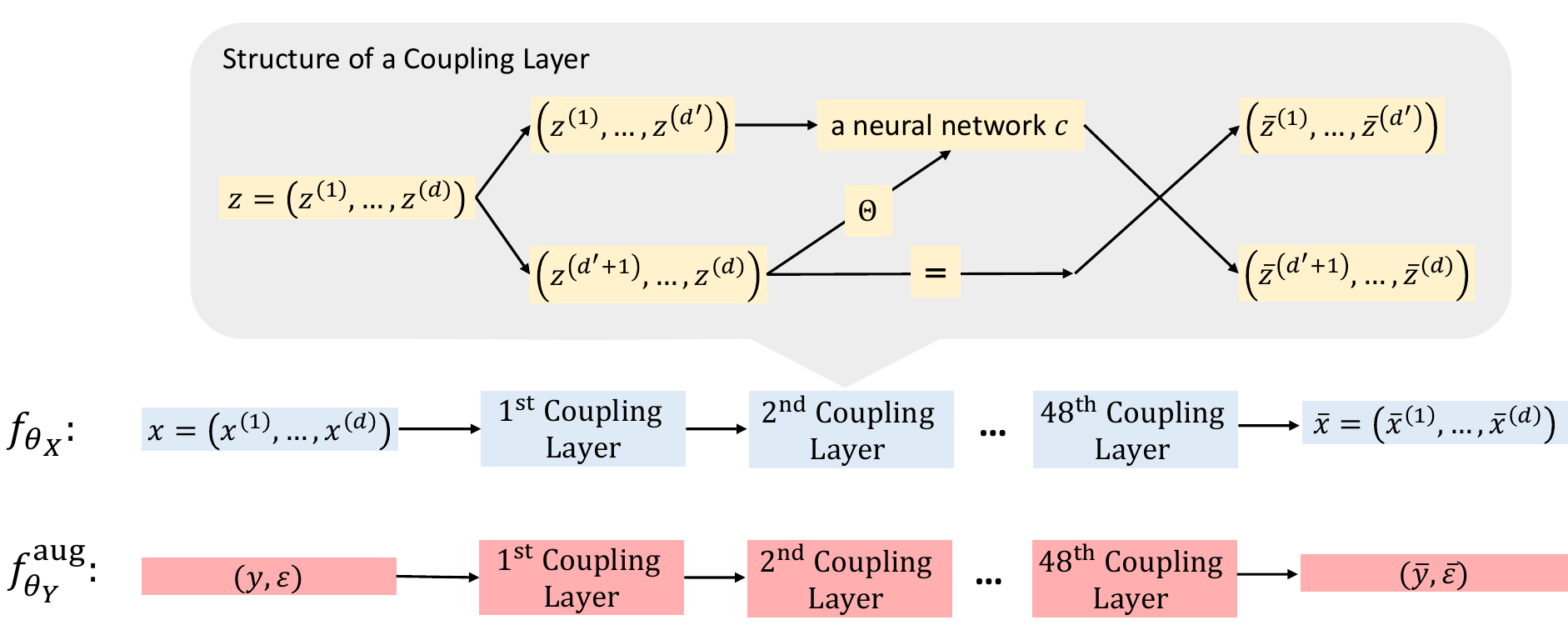}
  \caption{Illustration of the coupling layer structure and the overall composition of Augmented BNF.}
  \label{fig: architecture} 
\end{figure*}
\section{Conformalized quantile regression}\label{appendix: CQR}
The conditional guarantee in Eq.~(\ref{eq: conditional guarantee}) is not practically achievable using finite calibration samples without regularity assumptions, such as Lipschitz continuity of $P_{Y|x}$ density~\citep{foygel2021limits}. Hence, approximations of the conditional guarantee are extensively developed. Mondrian CP ensures $1-\alpha$ coverage conditioned over input subspaces~\citep{bostrom2021mondrian}. Some methods estimate the conformal score distribution conditioned on specific test input $x$, for example, by weighting each $V_i$ based on the proximity of $X_i$ to $x$~\citep{lin2021locally, guan2023localized, gibbs2023conformal}. Conformal training embeds a size-based loss in the training of the model $h$~\citep{correia2024information,stutz2021learning, bars2025volumeminimizationconformalregression}. Besides, advanced score functions are developed to facilitate conditional coverage in regression~\citep{romano2019conformalized, feldman2021improving}. Generative models also show promise for enhancing adaptiveness, especially for multivariate output~\citep{colombo2024normalizing, fang2025contra,klein2025multivariateconformalpredictionusing,thurin2025optimaltransportbasedconformalprediction}.

In this work, we apply Conformalized quantile regression (CQR)~\citep{romano2019conformalized} to approximate the conditional guarantee in Eq.~(\ref{eq: conditional guarantee}). CQR first trains two regression models with pinball loss at levels  $1-\alpha/2$ and $\alpha/2$, respectively, then calibrates the resulting intervals using residuals on a separate calibration set.

For clarity, we introduce CQR in the context of sample normalization and multi-source domain generalization.
For a regression model $h$, the pinball loss~\citep{steinwart2011estimating} at quantile level $\alpha$ for sample $(x,y)$ is defined as
\begin{equation}\label{eq:pinball loss}
        l_\alpha(h(x),y)= 
\begin{cases}
    \alpha\left(y-h(x)\right)& \text{if } y-h(x)>0,\\
    (1-\alpha)\left(h(x)-y\right)& \text{otherwise.}
\end{cases}    
\end{equation}
The models $h_\text{hi}$ and $h_\text{lo}$ are trained by optimizing the pinball loss in Eq~(\ref{eq:pinball loss}) at quantile levels $1-\alpha/2$ and $\alpha/2$, respectively. For calibration instances $\{(X_i,Y_i)\}_{i=1}^n$ drawn from $P_{XY}$, conformal scores are defined as 
\begin{equation}\label{eq: cqr score function}
    V_i=\max\left\{h_\text{lo}(X_i)-Y_i,Y_i-h_\text{hi}(X_i)\right\}\text{, for }i=1,...,n.
\end{equation}
Let $\tau$ be the ${\lceil(1-\alpha)(n+1)\rceil}/{n}$ quantile of $\{V_i\}_{i=1}^n$. If a test sample $(X_{n+1},Y_{n+1})\sim Q_{XY}$ is normalized to $(\widebar{X}_{n+1},\widebar{Y}_{n+1})\sim P_{XY}$, we construct an adaptive prediction set
\begin{equation}
    \mathcal{C}_\text{CQR}(\widebar{X}_{n+1})=\left[h_\text{lo}(\widebar{X}_{n+1})-\tau, h_\text{hi}(\widebar{X}_{n+1})+\tau\right].
\end{equation}
Here, $h_\text{lo}$ and $h_\text{hi}$ predict the likely lower and upper ends, while $\tau$ adjusts the set based on how well the predictions fit the calibration data.
As proved in~\citep{romano2019conformalized}, $\mathcal{C}_\text{CQR}$ can empirically approximate the conditional coverage guarantee described in Eq.~(\ref{eq: conditional guarantee of the normalized sample}). Extensions of CQR are explored in~\citep{kivaranovic2020adaptivedistributionfreepredictionintervals, Sesia_2020}, which modified the score function in Eq.~(\ref{eq: cqr score function}) for higher adaptiveness.

\section{A brief review of the Sinkhorn algorithm}\label{appendix: sinkhorn algorithm}
As we introduced in Definition~\ref{def: Wassertein distance between empirical distributions}, the Wasserstein distance between two empirical distributions $\widehat{\mu}_X$ and $\widehat{\nu}_X$ with $p=1$ is given by
\[
{W}(\widehat{\mu}_X, \widehat{\nu}_X) = \min_{\gamma \in \Gamma(\widehat{\mu}_X, \widehat{\nu}_X)} \sum_{i=1}^n \sum_{j=1}^m \gamma_{ij} \, C_{ij}.
\]
where $C \in \mathbb{R}^{n \times m}$ is the cost matrix with entries $C_{ij} = d_\mathcal{X}(x_i, x'_j)$, and $\Gamma(\widehat{\mu}_X, \widehat{\nu}_X)$ is the set of joint distributions $\gamma \in \mathbb{R}_+^{n \times m}$ with marginals $\widehat{\mu}_X$ and $\widehat{\nu}_X$.

To make this optimization problem more tractable, the Sinkhorn algorithm~\citep{cuturi2013sinkhorn} introduces an entropic regularization term:
\[
W^\beta(\widehat{\mu}_X, \widehat{\nu}_X) = \min_{\gamma \in \Gamma(\widehat{\mu}_X, \widehat{\nu}_X)} \sum_{i=1}^n \sum_{j=1}^m \gamma_{ij} \, C_{ij} + \beta \sum_{i=1}^n \sum_{j=1}^m \gamma_{ij} \log \gamma_{ij},
\]
where $\beta > 0$ controls the strength of the regularization.

This regularized objective is strictly convex and can be efficiently minimized via iterative matrix scaling. 
Let $K = \exp(-C / \beta)$ be the Gibbs kernel. The scaling vectors $u \in \mathbb{R}^n$ and $v \in \mathbb{R}^m$ are initialized to all ones and updated via
\[
u \gets\ \frac{1/n}{K v}, \quad v \gets \frac{1/m}{K^\top u},
\]
where divisions are element-wise. Once converged with small changes in $u$ and $v$, the optimal transport plan takes the form
\[
\gamma^\ast = \mathrm{diag}(u) \, K \, \mathrm{diag}(v).
\]

This approach yields a differentiable approximation to the true Wasserstein distance, enabling its integration into gradient-based optimization pipelines. We refer to~\citep{cuturi2013sinkhorn,knight2008sinkhorn,feydy2020analyse} for more detailed studies about the Sinkhorn algorithm.
\section{Introduction to baselines}\label{appendix: baselines}
Figure~\ref{fig: examples} highlights the distinctions between the baseline methods and the proposed approach. SCP constructs prediction sets of fixed size and ensures only marginal coverage under i.i.d. assumptions, rendering it ineffective under joint distribution shifts. IW-CP addresses only covariate shift and causes its prediction intervals to contract in the example, because test features are distributed in regions where calibration data is concentrated. WC-CP accounts for worst-case distribution shifts, expanding prediction sets until $1-\alpha$ marginal coverage is achieved on the test data, which can be inefficient. WR-CP improves upon this by regularizing the base predictive model through minimizing the Wasserstein distance between calibration and test conformal scores, producing more compact prediction sets while maintaining robust marginal coverage. All of these methods, however, focus exclusively on marginal coverage. CQR, a representative conditional conformal prediction method, fails to handle distributional shifts. In contrast, the Augmented BNF transformation model learns an invertible mapping between calibration and test data, enabling robust conditional coverage even under non-i.i.d. conditions.
\begin{figure*}[h]
\centering
\captionsetup{singlelinecheck = false, justification=justified}
\includegraphics[scale=0.35]{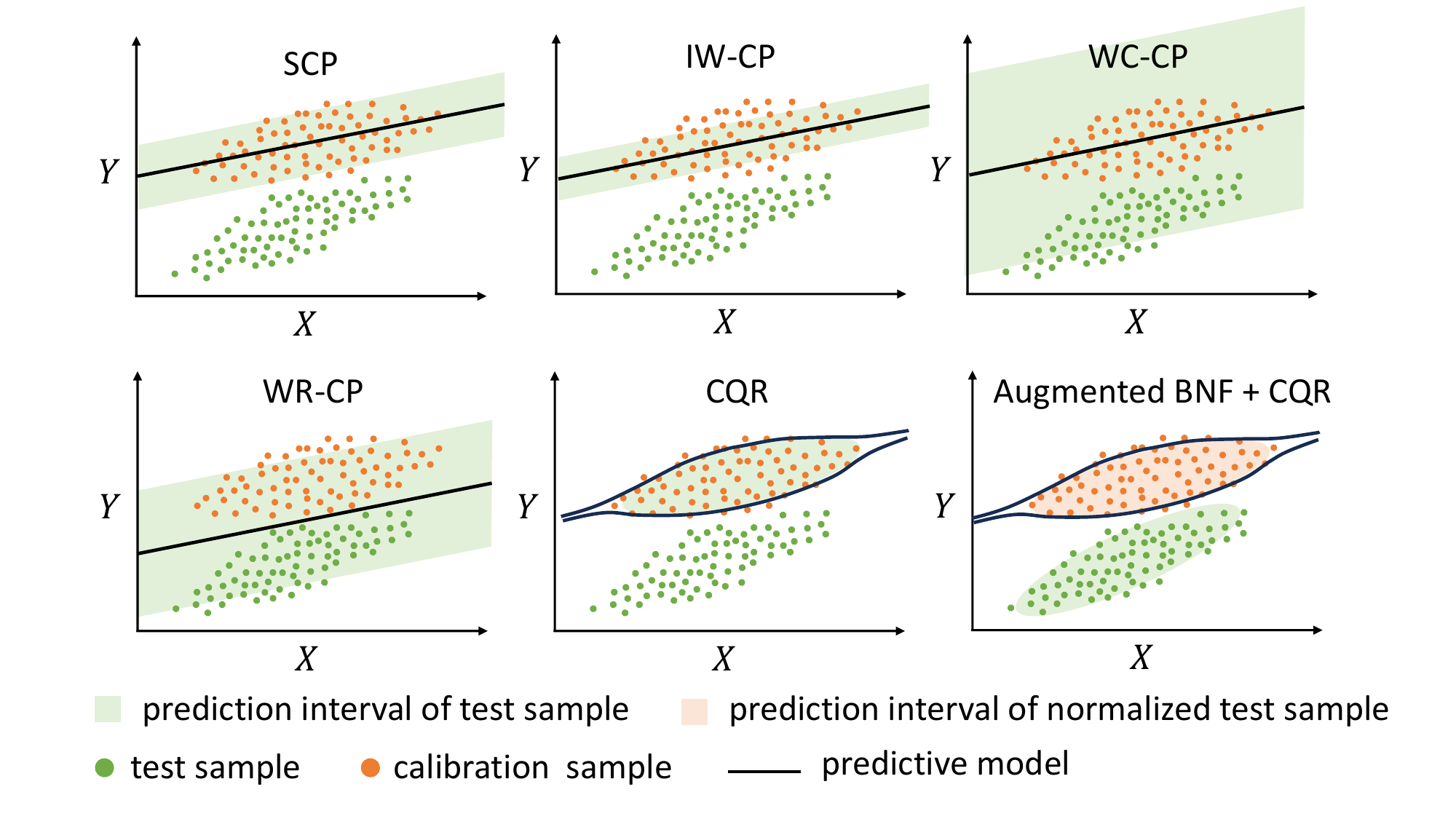}
\caption{Comparison between baselines and the proposed method via a toy example. Augmented BNF effectively transforms the prediction intervals from the calibration distribution to the test distribution.}
\label{fig: examples}
\end{figure*}
\section{Data preparation for multi-source domain generalization}\label{appendix: datasets}
We introduce the data preparation procedure shared across all datasets. We set \( K = 3 \), partitioning each dataset into three subsets, each exhibiting a distinct distribution shift. For each dataset, we conduct 10 independent sampling trials. In each trial, we first sample \(\mathcal{S}_{D^{k}}\) from subset \(k\) without replacement. Since calibration and training data typically share the same distribution in conformal prediction, \(\mathcal{S}_P\) is then sampled from the union of all \(K\) subsets, also without replacement. Finally, 100 different \(\mathcal{S}_Q\) sets are sampled as random mixtures from the remaining data. This procedure ensures that \(\mathcal{S}_{D^{k}}\) for \(k=1,\dots,K\), \(\mathcal{S}_P\), and \(\mathcal{S}_Q\) are mutually disjoint. Since the Sinkhorn algorithm is more numerically stable when comparing empirical distributions with matching sample sizes,  we set the calibration set and each training set to have equal sizes, i.e., $|\mathcal{S}_P| = |\mathcal{S}_{D^k}|$ for all $k = 1, \dots, K$. Experimental results are aggregated over the 10 trials for each dataset.

We also leverage a toy example from~\citep{xu2025wassersteinregularized} to demonstrate joint distribution shift under multi-source domain generalization in Figure~\ref{fig: random_mixture}.
\begin{figure*}[h]
\centering
\captionsetup{singlelinecheck = false, justification=justified}
\includegraphics[scale=0.60]{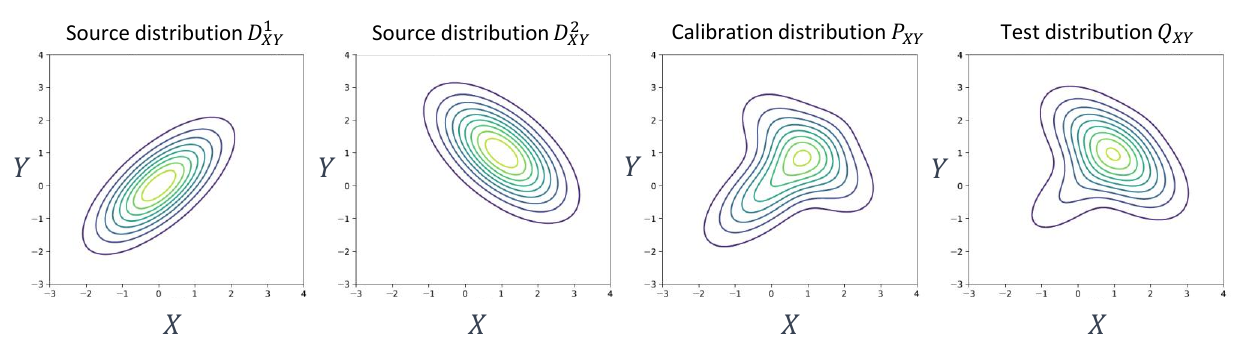}
\caption{\textbf{Multi-source domain generalization.} 
The test distribution \( Q_{XY} \) is a random mixture of source distributions, while the calibration distribution \( P_{XY} \) is a fixed known distribution. As a result, a distribution shift occurs since \( P_{XY} \neq Q_{XY} \).}
\label{fig: random_mixture}
\end{figure*}
\subsection{Synthetic distribution shifts}
The Physicochemical Properties of Protein Tertiary Structure (PTS) dataset~\citep{physicochemical_properties_of_protein_tertiary_structure_265} contains 45,730 instances, with the target variable being the protein decoy size. It includes nine features: surface area, non-polar exposed area, fractional area of exposed non-polar residue, fractional area of exposed non-polar part, molecular mass weighted exposed area, average deviation, Euclidean distance, secondary structure penalty, and spatial distribution constraints. Raw data is split into three subsets based on the distribution of the secondary structure penalty, thereby introducing distribution shifts among the subsets. We also use the PTS dataset to perform ablation studies on the approximation ability (with varying sample sizes) and generalization performance (with different numbers of source domains) of Augmented BNF.

\subsection{Natural distribution shifts reflecting real-world challenges}

\textbf{Generalized sales prediction over time-series} is crucial for risk-averse business decision-making~\citep{jin2022domain}. Moreover, sales data typically exhibit strong periodic patterns, such as seasonal or weekly fluctuations. Thus, effectively utilizing data from each sub-period to model a robust and generalized sales pattern is critical for achieving reliable forecasts. This requires models not only to capture short-term variations but also to generalize across different temporal domains, where distribution shifts may occur naturally due to changes in consumer behavior, external events, or market conditions. We consider the Bike Rental dataset~\citep{bike_sharing_275} to reflect this challenge. The dataset records hourly and daily rental counts from the Capital Bikeshare system during 2011 and 2012, along with associated weather and seasonal information. We partition the data based on rental hours into three time intervals: [0,8] (midnight), [9,16] (daytime), and [17,23] (evening). For prediction, we select continuous features including temperature, feeling temperature, humidity, and wind speed. The target variable is the count of rental bikes.

\textbf{Traffic speed prediction with mismatched data} focuses on transferring models trained on source distributions (e.g., traffic patterns on regular days and at major intersections) to test distributions exhibiting different characteristics (e.g., traffic patterns on special days and at minor intersections). For example, recent work has proposed traffic-law-informed models based on reaction-diffusion equations to provide generalized speed predictions~\citep{sun2023reaction}. Nevertheless, enhancing the reliability of uncertainty quantification under such distribution shifts remains a significant challenge. The Seattle-Loop dataset contains traffic volume and speed data collected in Seattle throughout 2015, recorded by sensors at 5-minute intervals~\citep{cui2019traffic}. PEMSD4 includes traffic data from 29 roads in San Francisco collected between January and February 2018, while PEMSD8 covers 8 roads in San Bernardino from July to August 2016~\citep{bai2020adaptive}. The task is to predict traffic speed at the next time step based on current speed and volume measurements. With $K=3$, we select one representative intersection from each dataset. Due to varying local traffic patterns, natural distribution shifts arise among the three locations. Our goal is to achieve strong generalization across these locations, ensuring robust predictions on any test sites where traffic patterns resemble a random mixture of the three selected intersections.

\textbf{Fair medical decision-making for patients from different hospitals} is essential for ensuring equitable healthcare outcomes. Variations in patient demographics, medical imaging scanners, laboratory equipment, and clinical practices across hospitals can lead to distribution shifts in the data. This phenomenon is commonly referred to as the multi-center issue~\citep{das2022multicenter}. Addressing this challenge is essential for building predictive models that remain accurate and fair across diverse healthcare institutions~\citep{olsson2022estimating}. To validate the effectiveness of the proposed method in this task, we collect patient data from a collaborating hospital. Additionally, we use the MIMIC-IV~\citep{johnson2023mimic} and eICU~\citep{pollard2018eicu} datasets to simulate data from two other hospitals. The goal is to fairly predict patients' ICU stay times based on their Apache scores and blood urea nitrogen (BUN) levels, ensuring reliable performance regardless of which center a patient originates from. We consider fair medical prediction to be achieved across the three data sources if the model exhibits comparable performance on random mixtures of the sources.
 
\textbf{Robust epidemic modeling across pandemic phases} can facilitate timely public health responses and resource planning. The U.S. Centers for Disease Control and Prevention (CDC) categorizes an epidemic period into three main phases: initiation, acceleration, and deceleration~\citep{cdcPandemicIntervals}. Each of these phases exhibits distinct epidemiological characteristics, which lead to natural distribution shifts. Traditional forecasting methods typically rely on Susceptible-Infectious-Recovered (SIR) models~\citep{harko2014exact,kabir2019analysis,turkyilmazoglu2022restricted} to predict the number of recently infected patients, aiming for robustness across the different pandemic phases. We demonstrate the application of the proposed method using the U.S. Influenza-like Illness (ILI) dataset~\citep{deng2020cola}, which contains weekly reports from the CDC on the number of ILI patients. The objective is to predict new infections for the upcoming week using both the weekly increase of infected patients and the cumulative infections for the year. The raw data is divided into three subsets based on the corresponding pandemic phases. We consider a forecasting model to be robust if its predictions remain reliable on random mixtures of data from the three phases.

We further apply t-SNE~\citep{van2008visualizing} to map the samples from each source into two dimensions, as shown in Figure~\ref{fig: data}. The visualization reveals clear distributional shifts between most sources. However, for some cases, such as the second and third sources in the Bike Rental and Fair Med setups, the distributions appear more similar. This slight overlap is not the result of manually creating similar data but arises naturally from the datasets themselves, which are collected from real-world scenarios.

\begin{figure}[H]
\centering
\captionsetup{singlelinecheck = false, justification=centering}
  \includegraphics[scale=0.6]{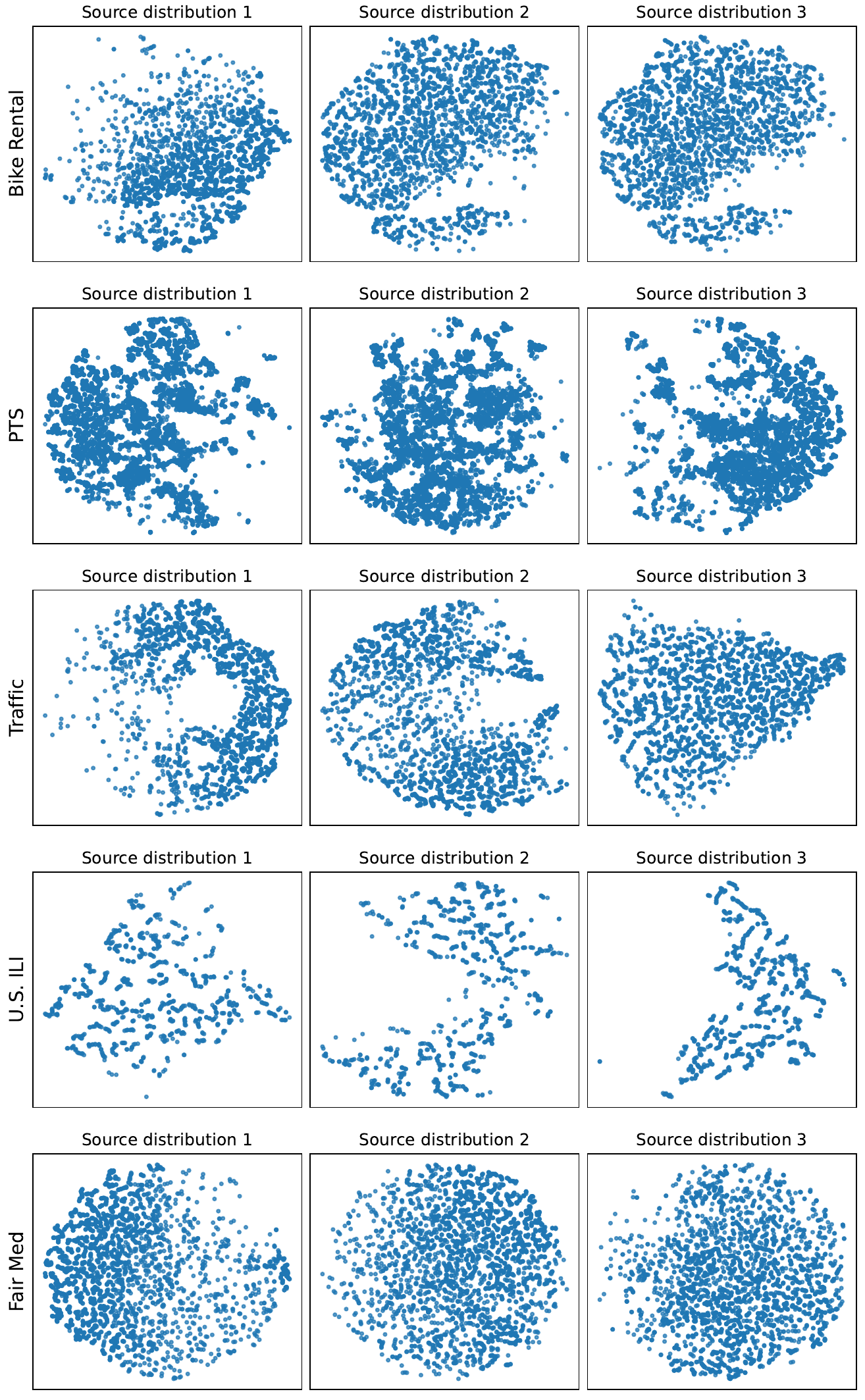}
  % \vspace{-15pt}
  \caption{Empirical data distributions of source domains after applying t-SNE.}
  \label{fig: data} 
\end{figure}
\section{Worst-slice coverage (WSC)}\label{appendix: wsc}
Worst-slice coverage (WSC)~\citep{cauchois2021knowing} quantifies the minimum empirical coverage over any slab \( \mathcal{S} \subseteq \mathcal{X} \) that contains at least 10\% of the test samples in \( \mathcal{S}_Q \). Specifically, for any CP methods that produce a prediction set $C(x)$ given an input $x$, WSC is defined by
\begin{equation}\label{eq: WSC}
    \text{WSC}=\inf_{\mathcal{S} \subseteq \mathcal{X}} \Pr(y\in C(x)|x\in\mathcal{S}), \text{s.t.} \Pr(x\in\mathcal{S}|(x,y)\in\mathcal{S}_Q)\geq 0.1.
\end{equation}
Nevertheless, WSC only evaluates the infimum slice coverage and therefore fails to penalize over-coverage. As also noted in~\citep{romano2020classification}, ensuring a high worst-case slice does not guarantee good conditional coverage across $\mathcal{X}$, particularly when different regions exhibit excessive coverage.

These limitations motivate our introduction of WSCG in Eq.~(\ref{eq: WSCG}), which penalizes both under- and over-coverage from the target level $1-\alpha$. As a result, WSCG provides a more comprehensive and balanced assessment of conditional coverage robustness.

\section{Generalization performance of Augmented BNF}\label{appendix: approximation ability}
\subsection{Various number of source domains}\label{appendix: generalization of various K}
\begin{wrapfigure}[12]{r}{6cm}
\centering
\captionsetup{singlelinecheck = false, skip=5pt, justification=justified}
\vspace{-10pt}
  \includegraphics[width=0.4\textwidth]{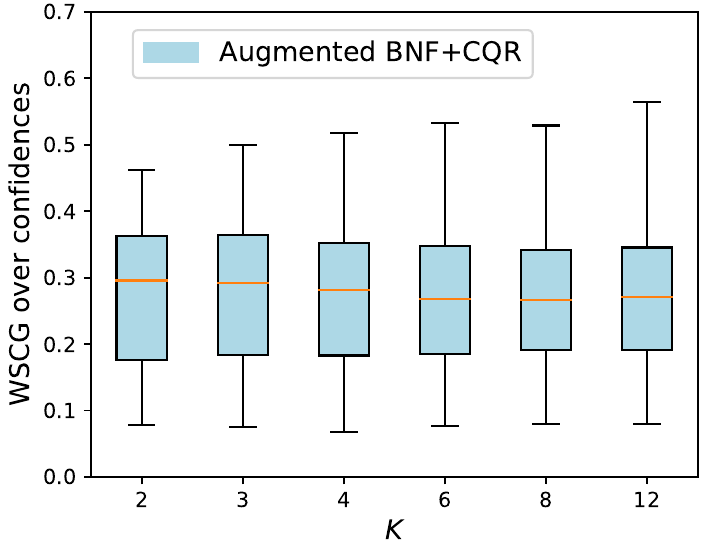}
  \vspace{-5pt}
  \caption{
  Generalization performance with different numbers of source domains.}
  \label{fig: env_ablation} 
\end{wrapfigure}
To explore the generalization ability of Augmented BNF under varying numbers of source domains, we modified the sampling procedure in Appendix~\ref{appendix: datasets} by changing \( K \in \{2, 3, 4, 6,8, 12\} \). For each value of \( K \), we generated 10 independent trials using the PTS dataset to account for sampling variability. Augmented BNF combined with CQR was applied to each trial across confidence levels \(1 - \alpha \in [0.1, 0.9]\), enabling a comprehensive evaluation.
Figure~\ref{fig: env_ablation} shows that increasing the number of source domains does not significantly degrade conditional coverage robustness, suggesting that Augmented BNF generalizes well even in the presence of greater domain heterogeneity.

\subsection{Different sample sizes}
\begin{wrapfigure}[12]{r}{6cm}
\centering
\captionsetup{singlelinecheck = false, skip=5pt, justification=justified}
\vspace{-10pt}
  \includegraphics[width=0.4\textwidth]{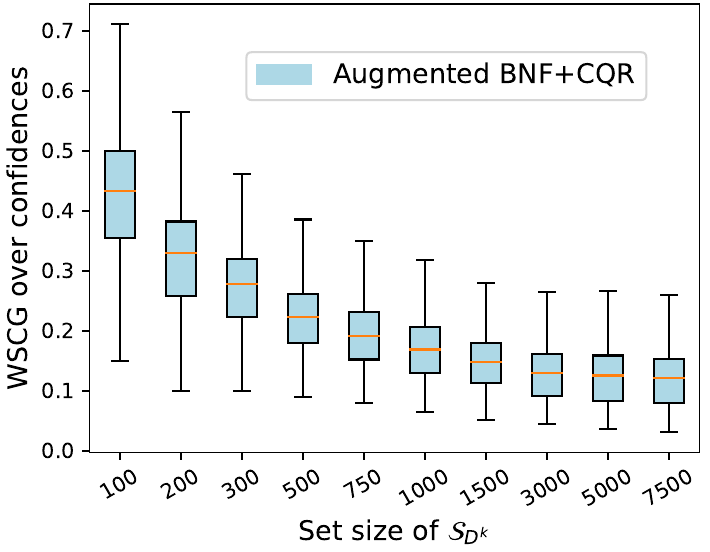}
  \vspace{-5pt}
  \caption{
  Impact of data availability on Augmented BNF approximation ability.}
  \label{fig: sample_ablation} 
\end{wrapfigure}
Generative models may struggle to approximate underlying distributions when data are limited, especially in high dimensions~\citep{kong2020expressivepowerclassnormalizing, poggio2017and}.
To assess how sample size affects the performance of Augmented BNF, we vary the number of samples in $\mathcal{S}_{D^k}$ and perform 10 trials for each setting on the PTS dataset. For each trial, we apply Augmented BNF+CQR across $1-\alpha$ from 0.1 to 0.9 and compute the WSCG over all confidence levels. As shown in Figure~\ref{fig: sample_ablation}, WSCG degrades clearly when the size drops below 750. With 11-dimensional data on PTS, it gives a sample-to-dimension ratio of ~68, demonstrating that our method is reasonably data-efficient.
\section{Coverage lower bounds under imperfect transformation}\label{appendix: lower bound}
\subsection{Marginal coverage lower bound}

We establish a marginal coverage lower bound by quantifying the alignment between the calibration conformal scores and those obtained from the transformed test distribution.

Marginal coverage gap can be defined as the discrepancy between the CDFs of $P_V$ and $Q_V$ at the calibration quantile $\tau$~\citep{xu2025wassersteinregularized}. After applying the transformation $f_\theta^{\text{aug}}$, the test distribution $Q_{XY}$ is mapped to $\widebar{Q}_{XY}:={{f_{\theta}^{\text{aug}}}}_\# Q_{XY}$, yielding the conformal score distribution $\widebar{Q}_{V}:=s_\#\widebar{Q}_{XY}$, where $s$ denotes the score function. The residual marginal coverage gap after transformation is therefore
\begin{equation}
    |F_V^P(\tau)-F_{V}^{\widebar{Q}}(\tau)|,
\end{equation}
with $F$ denoting the CDF.

This leads to the following lower bound on the marginal coverage of prediction sets produced by the Augmented BNF transformation model:
\begin{equation}
    \Pr(Y_{n+1}\in \mathcal{C}_\text{BNF}^\text{aug}(X_{n+1}))\geq1-\alpha-|F_V^P(\tau)-F_{V}^{\widebar{Q}}(\tau)|.
\end{equation}
Within the multi-source domain generalization (MSDG) framework, the test distribution $Q_{XY}$ is assumed to be a random mixture of source distributions $\{D_{XY}^k\}_{k=1}^K$. Denoting $\widebar{D}_{XY}^k:={{f_{\theta}^{\text{aug}}}}_\# D_{XY}^k$ and $\widebar{D}^k_{V}:=s_\#\widebar{D}_{XY}^k$, we can bound the marginal coverage gap as
\begin{equation}
    |F_V^P(\tau)-F_{V}^{\widebar{Q}}(\tau)|\leq\sup\nolimits_{k\in\{1,...,K\}}|F_V^P(\tau)-F_{V}^{\widebar{D}^k}(\tau)|.
\end{equation}
Consequently, we obtain the final marginal coverage lower bound under MSDG as
\begin{equation}\label{eq: marginal lower bound}
     \Pr(Y_{n+1}\in \mathcal{C}_\text{BNF}^\text{aug}(X_{n+1}))\geq1-\alpha-\sup\nolimits_{k\in\{1,...,K\}}|F_V^P(\tau)-F_{V}^{\widebar{D}^k}(\tau)|.
 \end{equation}
For validation, we compare the theoretical bound with the empirical marginal coverage observed across randomly sampled test distributions. The results, presented in Figure~\ref{fig: marginal lower bound}, show that the empirical coverage for most test distributions exceeds the proposed lower bound, thereby confirming the validity of Eq.~(\ref{eq: marginal lower bound}). The closeness between $1-\alpha$ and the bound further indicates that our method effectively aligns the calibration and test distributions.
\begin{figure*}[h]
\centering
\captionsetup{singlelinecheck = false, justification=centering}
\includegraphics[scale=0.4]{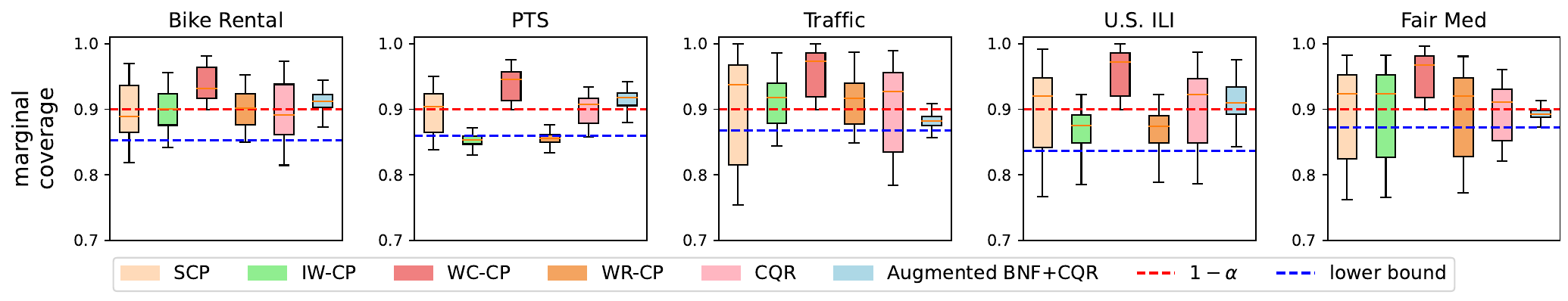}
\caption{Marginal coverage achieved by $\mathcal{C}_\text{BNF}^\text{aug}(X_{n+1})$ compared with the proposed marginal lower bound.}
\label{fig: marginal lower bound}
\end{figure*}

\subsection{Conditional coverage lower bound}
Next, we establish a conditional coverage lower bound that accounts for the imperfect alignment between calibration and test data induced by Augmented BNF.

First, even in the i.i.d. setting, exact conditional coverage is unattainable with finite samples~\citep{vovk2012conditional,lei2014distribution,foygel2021limits}. For instance,~\cite{romano2019conformalized} explicitly note that CQR achieves conditional coverage only on the training data, not on unseen test samples. Likewise, the performance of LCP~\citep{guan2023localized} is highly sensitive to the choice of kernel bandwidth, preventing finite-sample conditional coverage guarantees.
Consequently, even under the i.i.d. assumption such that $(X_{n+1},Y_{n+1})\sim P_{XY}$, the application of CQR can only guarantee that
\begin{equation}
    \Pr(Y_{n+1}\in \mathcal{C}_\text{CQR}(X_{n+1})|X_{n+1}=x)\geq 1-\alpha-\alpha_\text{i.i.d.},
\end{equation}
where $\alpha_\text{i.i.d.}$ reflects the approximation error introduced by CQR. This gap is an intrinsic limitation of existing conditional CP approaches.

Secondly, under distribution shift, for a test sample $(X_{n+1},Y_{n+1})\sim Q_{XY}$, we combine the approximation error $\alpha_\text{i.i.d.}$ from CQR and the lower bound in Eq.~(\ref{eq: lowerbound by CCG}) and derive
\begin{equation}
    \Pr\left(Y _{n+1}\in \mathcal{C}_\text{CQR}(X _{n+1})|X _{n+1}=x\right)\geq 1-\alpha-\alpha_\text{i.i.d.}-\text{CCG}(P,Q,x)
\end{equation}
To evaluate the expected conditional coverage across the test distribution, we take the expectation over $x\sim Q_X$ and obtain
\begin{equation}
    \mathbb{E}_{x\sim Q_X}[\Pr(Y _{n+1}\in \mathcal{C}_\text{CQR}(X _{n+1})|X _{n+1}=x)]\geq 1-\alpha-\alpha_\text{i.i.d.}-\text{ICG}(P,Q),
\end{equation}
where the $\text{ICG}(P,Q)$ is defined in Eq.~(\ref{eq: integrated coverage gap}) as the expectation of $\text{CCG}(P,Q,x)$ over $Q_X$.

Using our bound on $\text{ICG}(P, Q)$ in terms of the Wasserstein distance $W(P_{XY}, Q_{XY})$ in Eq.~(\ref{eq: final bound of ICG}), we obtain a bound on the expected conditional coverage under distribution shift:
\begin{equation}
\begin{split}
        &\mathbb{E}_{x\sim Q_X}[\Pr(Y _{n+1}\in \mathcal{C}_\text{CQR}(X _{n+1})|X _{n+1}=x)]\\&\geq 1-\alpha-\alpha_\text{i.i.d.}-\sqrt{2\kappa L}\left(\eta\cdot W(P _{XY},Q _{XY})+1/4\right).
\end{split}
\end{equation}
Finally, the transformation by Augmented BNF lead to a more robust prediction set $\mathcal{C}_\text{BNF}^\text{aug}(X_{n+1})$. Letting $\widebar{Q}_{XY}$ be ${{f_{\theta}^{\text{aug}}}}_\# Q_{XY}$, we clarify the role of the remaining distance $W(P_{XY},\widebar{Q}_{XY})$ by
\begin{equation}
\begin{split}
        &\mathbb{E}_{x\sim Q_X}[\Pr(Y _{n+1}\in \mathcal{C}_\text{BNF}^\text{aug}(X _{n+1})|X _{n+1}=x)]\\&\geq 1-\alpha-\alpha_\text{i.i.d.}-\sqrt{2\kappa L}\left(\eta\cdot W(P_{XY},\widebar{Q}_{XY})+1/4\right).
\end{split}
\end{equation}
We denote $\alpha_\text{trans}:=\sqrt{2\kappa L}(\eta\cdot W(P_{XY},\widebar{Q}_{XY})+1/4)$ to quantify the remaining deviation induced by imperfect alignment between calibration and test distributions.  Figure~\ref{fig: approximation} shows that the proposed transformation model effectively approximates the CQR under the i.i.d. condition. This suggests that $\alpha_\text{trans}$ is significantly smaller than $\alpha_\text{i.i.d.}$ with the remaining coverage gap primarily attributable to the approximation error of CQR itself.
\section{Prediction efficiency under source conditioning}\label{appendix: prediction efficiency}
\subsection{Prediction inefficiency by Augmented BNF}
Augmented BNF uses Eq.~(\ref{eq: aug BNF prediction set}) to obtain prediction sets on the test distribution $Q_{XY}$. During training, this augmented component $\varepsilon$ of the $Y$ branch $f_{\theta_Y}^\text{aug}$ in Eq.~(\ref{eq: normalization via augmented BNF}) is sampled from a single Gaussian distribution $\mathcal{N}(0,1)$, making it independent of the training sample sources. As a result, the model learns a shared transformation for all training distributions $D^k_{XY}$ for $k=1,..., K$ to align with the calibration distribution 
$P_{XY}$.

At test time, this design leads to a key limitation: since $\varepsilon_{n+1}$ is source-agnostic, the $Y$ branch $f_{\theta_Y}^\text{aug}$ cannot infer which source distribution a new test sample originates from. As a result, the prediction set $\mathcal{C}_\text{BNF}^\text{aug}(X_{n+1})$ must widen to account for all sources to ensure valid coverage. This behavior corresponds to a conditional worst-case strategy and inherently results in larger prediction sets.

The prediction inefficiency is reflected in Figure~\ref{fig: coverage_size}. The Augmented BNF produces noticeably larger prediction sets on Traffic, U.S. ILI, and Fair Med. This is due to the substantial variation in the conditional label distributions $D^k_{Y|x}$ within the three settings, such as significantly different supports, which leads to enlarged prediction sets.

\begin{figure*}[h]
\centering
\captionsetup{singlelinecheck = false, justification=justified}
  \includegraphics[scale=0.4]{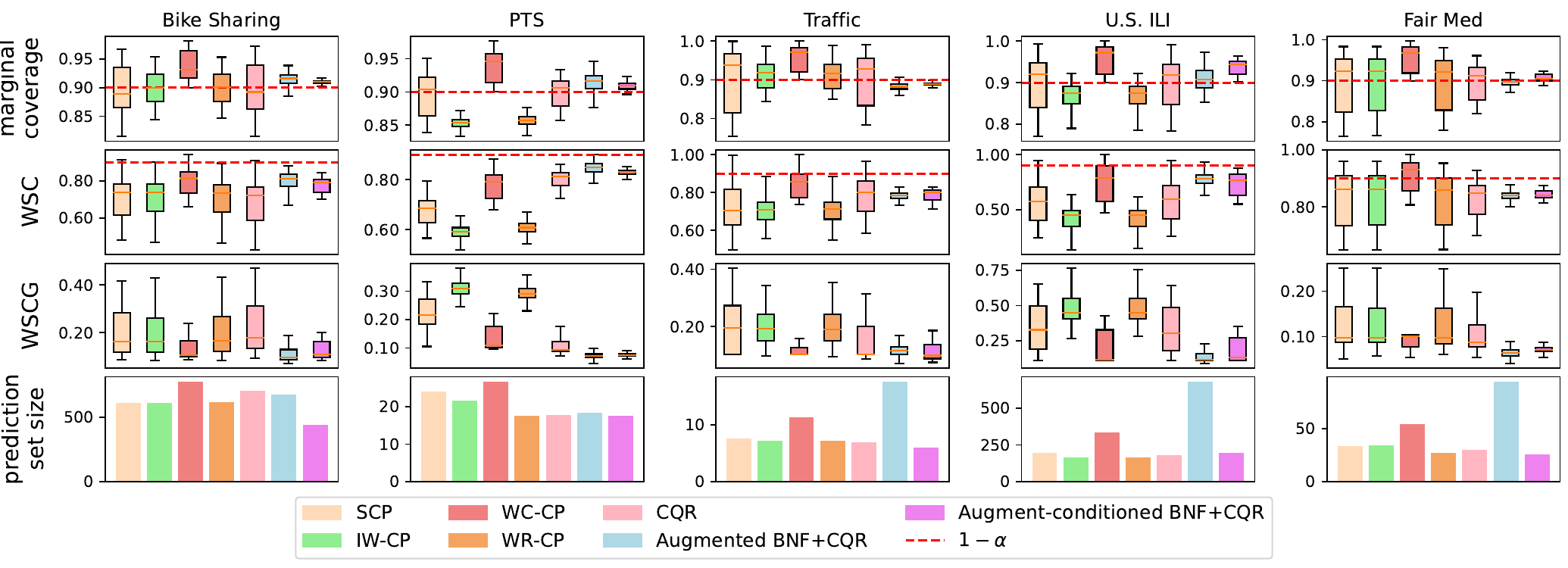}
  % \vspace{-15pt}
  \caption{Prediction set size comparison. The standard Augmented BNF can produce large prediction sets, whereas the Augment-Conditioned variant significantly reduces set size while preserving coverage performance.}
  \label{fig: coverage_size} 
  % \vspace{-15pt}
\end{figure*}
\subsection{Efficient prediction by Augment-Conditioned BNF}
To achieve smaller prediction sets, we propose \textbf{Augment-Conditioned BNF}, denoted as $f_\theta^\text{aug-cond}$. In this design, the augmented component $\varepsilon$ is sampled from a distinct Gaussian distribution $\mathcal{N}^k$ if an instance $(x,y)$ is from $D^{k}_{XY}$ during training. In other words, $\varepsilon$ serves not only to enhance expressiveness but also as a conditioning variable in $f_{\theta_Y}^\text{aug-cond}(y,\varepsilon)$. Consequently, during inference, if $\varepsilon_{n+1}$ correctly captures the source of the test sample, we can construct a smaller prediction set $\mathcal{C}_\text{BNF}^\text{aug-cond}(X_{n+1})$. Using a similar calculation as in Eq.~(\ref{eq: aug BNF prediction set}), we derive
\begin{equation}
        \mathcal{C}_\text{BNF}^{\text{aug-cond}}(X_{n+1})=\left\{y:f^{\text{aug-cond}}_{\theta_Y}(y;\varepsilon_{n+1})\in \mathcal{C}_\text{C}(\widebar{X}_{n+1})\right\}, \text{ where } \widebar{X}_{n+1}=f_{\theta_X}(X_{n+1})
\end{equation}
As shown in Figure~\ref{fig: coverage_size}, Augment-Conditioned BNF leads to significantly smaller prediction sets compared to the standard Augmented BNF. \cite{correia2024information} also demonstrates that knowledge of the test sample's source can serve as valuable side information to improve prediction efficiency. Such side information is often available in real-world applications. For instance, in multi-center healthcare settings, models are required to generalize across different hospitals. In these scenarios, the center at which a patient is admitted is typically known and can be used to tailor the prediction procedure. Leveraging this information allows the model to generate tighter prediction sets while maintaining valid coverage guarantees.

We observe that the Augment-Conditioned BNF may exhibit slightly higher WSCG compared to the standard Augmented BNF, such as on U.S. ILI. That is attributed to data sparsity, caused by the use of more distinct Gaussian distributions to model the data sources. As the number of Gaussians increases, the samples become more dispersed, making it more challenging to learn the underlying population distributions effectively. Hence, although the Augment-Conditioned BNF yields smaller prediction sets, it compromises robustness in coverage, revealing an inherent trade-off.
\section{Additional experiment results}\label{appendix: additional exp result}
\subsection{Ablation study across confidence levels}\label{appendix: alpha ablation}
Rather than fixing $1-\alpha=0.9$, we explore the performance of BNF across different confidence levels. After averaging over all datasets and trials, WSCG of Augmented BNF+CQR and five baselines with $1-\alpha$ from 0.1 to 0.9 are illustrated in Figure~\ref{fig: alpha ablation}. It is shown that BNF consistently achieves the most robust conditional coverage with the lowest mean WSCG as the confidence level varies. Moreover, we can observe a pattern across all methods where WSCG tends to be higher in the middle confidence range and lower at both ends of the spectrum. At a high confidence level (i.e. $1-\alpha=0.9$), prediction sets are so broad that they encompass the majority of possible outcomes, minimizing the potential for coverage gaps. Consequently, WSCG tends to be lower at high confidence levels. Similarly, prediction intervals are intentionally narrow at a low confidence level (i.e. $1-\alpha=0.1$), capturing only a small subset of outcomes. This makes them inherently less sensitive to distribution shifts, thereby reducing coverage gaps. Hence, the coverage gap is typically most pronounced at intermediate confidence levels, forming arch-like curves across the confidence spectrum in Figure~\ref{fig: alpha ablation}.
\begin{figure}[h]
\centering
\captionsetup{singlelinecheck = false, skip=5pt, justification=justified}
\vspace{-10pt}
  \includegraphics[width=0.6\textwidth]{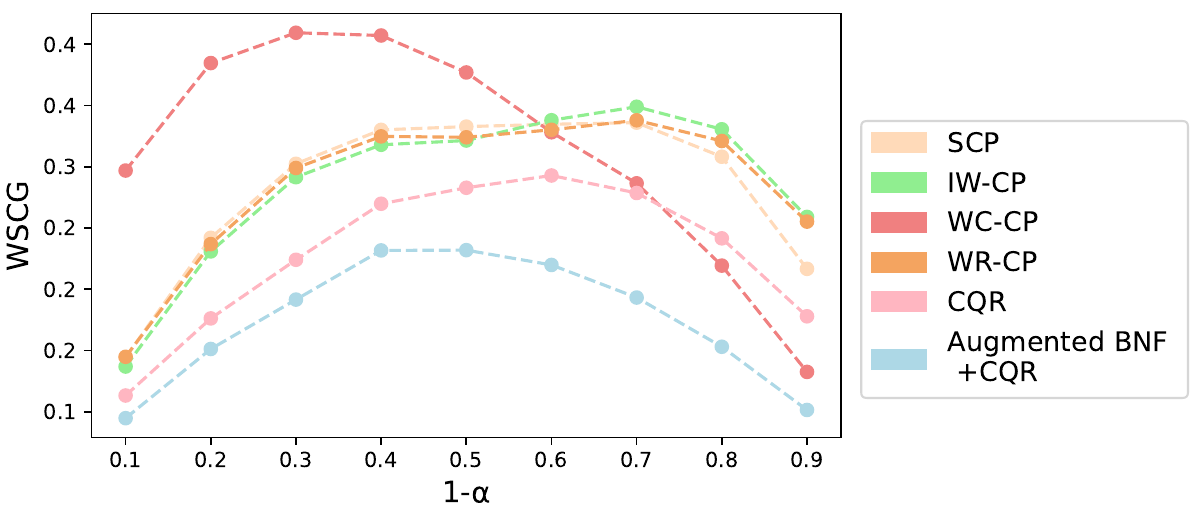}
  \vspace{-5pt}
  \caption{
  WSCG of Augmented BNF + CQR and baselines for $1-\alpha \in [0.1, 0.9]$, averaged across datasets.}
  \label{fig: alpha ablation} 
\end{figure}
\subsection{Challenges of feature conditioning in high dimensions}
We present the experiment result of  Feature-Conditioned BNF introduced in Section~\ref{sec: feature conditioning} across all dataset. 
This variant explicitly captures the dependency of $Y$ on $X$. However, it increases the input dimension of $f_{\theta_Y}^\text{fea}$ to $d+1$, making the total input dimension of Feature-Conditioned BNF $2d+1$, where $d$ is the dimension of $X$. As a result, the curse of dimensionality is exacerbated with a small sample-to-dimension ratio $|\mathcal{S}_{D^{k}}|/(2d+1)$, making true distributions harder to estimate. In contrast, Augmented BNF maintains a more favorable ratio of $|\mathcal{S}_{D^k}| / (d+2)$, as the input to $f_{\theta_Y}^{\text{aug}}$ is only two-dimensional. Consequently, with limited data, Feature-Conditioned BNF tends to yield higher WSCG due to poor approximation. We report the small sample-to-dimension ratio of Feature-Conditioned BNF in Table~\ref{Table: sample-to-dimension ratio}, and demonstrate its less robust conditional coverage in Figure~\ref{fig: feature}, where its WSCG tend to be higher than that of Augmented BNF.

Meanwhile, although a one-dimensional projection may alleviate this issue, such dimensionality reduction inevitably discards information that can be crucial for accurately modeling the conditional calibration distribution, and its deterministic nature prevents it from achieving the level of expressiveness offered by Augmented BNF. 
\begin{table}[h]
\centering
\footnotesize
\captionsetup{justification=centering}
\caption{Feature-Conditioned BNF holds a \textbf{small} sample-to-dimension ratio $|\mathcal{S}_{D^{k}}|/(2d+1)$.}
\def\arraystretch{1.2}
\begin{tabular}{c|c|c|c|c|c|c}
\toprule\toprule
Dataset & Bike Rental& PTS & Traffic & U.S. ILI & Fair Med\\ \midrule
$d$ & 4 & 9 & 3 & 2 & 2 \\ 
$|\mathcal{S}_{D^{k}}|$ & 2800 & 7500 & 2800 & 870 & 3000   \\ 
$|\mathcal{S}_{D^{k}}|/(d+2)$ & 466.7 & 681.8	& 560.0 &217.5 &750.0 \\ 
$|\mathcal{S}_{D^{k}}|/(2d+1)$ & \textbf{311.1} & \textbf{394.7} & \textbf{400.0} & \textbf{174.0} &  \textbf{600.0}\\ \bottomrule
\end{tabular}
\label{Table: sample-to-dimension ratio}
\end{table}

\begin{figure}[h]
\centering
\captionsetup{singlelinecheck = false, skip=5pt, justification=justified}
  \includegraphics[width=1.0\textwidth]{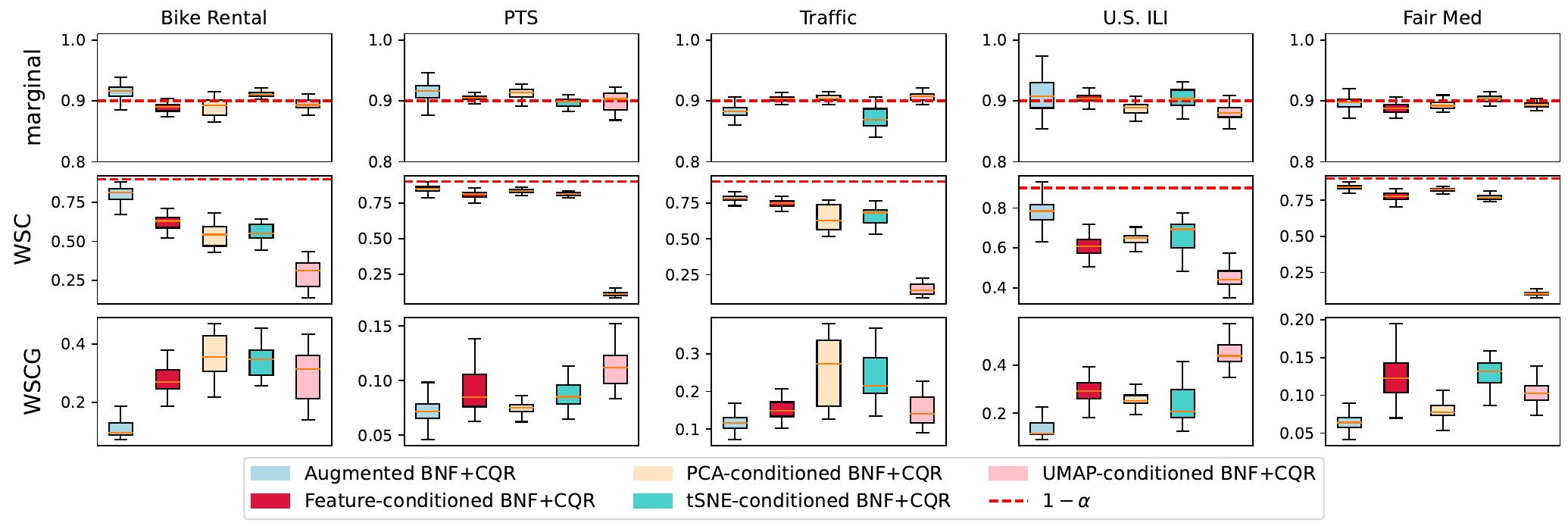}
  \vspace{-5pt}
  \caption{
  Comparison of conditioning the transformation of label $Y$ on the feature and its one-dimensional $X$ projection obtained via PCA, t-SNE, and UMAP.}
  \label{fig: feature} 
\end{figure}

\subsection{Additional Distribution Shift via Label Perturbation}
\label{appendix: label perturbation}

In Section~\ref{sec: perturbation}, we evaluate the robustness of our method beyond the multi-source setting by introducing distribution shift based on label perturbation~\cite{sesia2023adaptive, einbinder2022conformal}. Specifically, we consider a family of shifted distributions defined as
\[
Q_{XY}\in\left\{ 
D_{XY} : (X,Y)\sim P_{XY},\ \delta \sim \mathcal{U}(1,1.5),\ \tilde{Y} = Y\delta,\ (X,\tilde{Y}) \sim D_{XY}
\right\}.
\]
This formulation corresponds to a multiplicative perturbation applied to the label, where the scaling factor $\delta$ is sampled independently from a uniform distribution. As a result, the marginal distribution of $X$ remains unchanged, while the conditional distribution of $Y \mid X$ is systematically perturbed.

Compared to additive perturbations, this scale-based transformation induces a more pronounced and heterogeneous shift, particularly for larger label values, thereby providing a challenging testbed for evaluating conditional coverage.

In our experiments, we construct ten perturbed environments for training by setting $\delta \in \{1, 1.05, \ldots, 1.45, 1.5\}$, each corresponding to an independent realization of the scaling variable. The BNF is then trained to transport these shifted distributions back to the original (unperturbed) calibration distribution $P_{XY}$. During inference, we further evaluate the model on 100 randomly generated shifted distributions sampled from the same perturbation family.

The results across all datasets are presented in Figure~\ref{fig: perturbation}, where our method demonstrates a favorable trade-off between robust conditional coverage and prediction efficiency. 

\begin{figure}[h]
\centering
\captionsetup{singlelinecheck = false, skip=5pt, justification=justified}
  \includegraphics[width=1.0\textwidth]{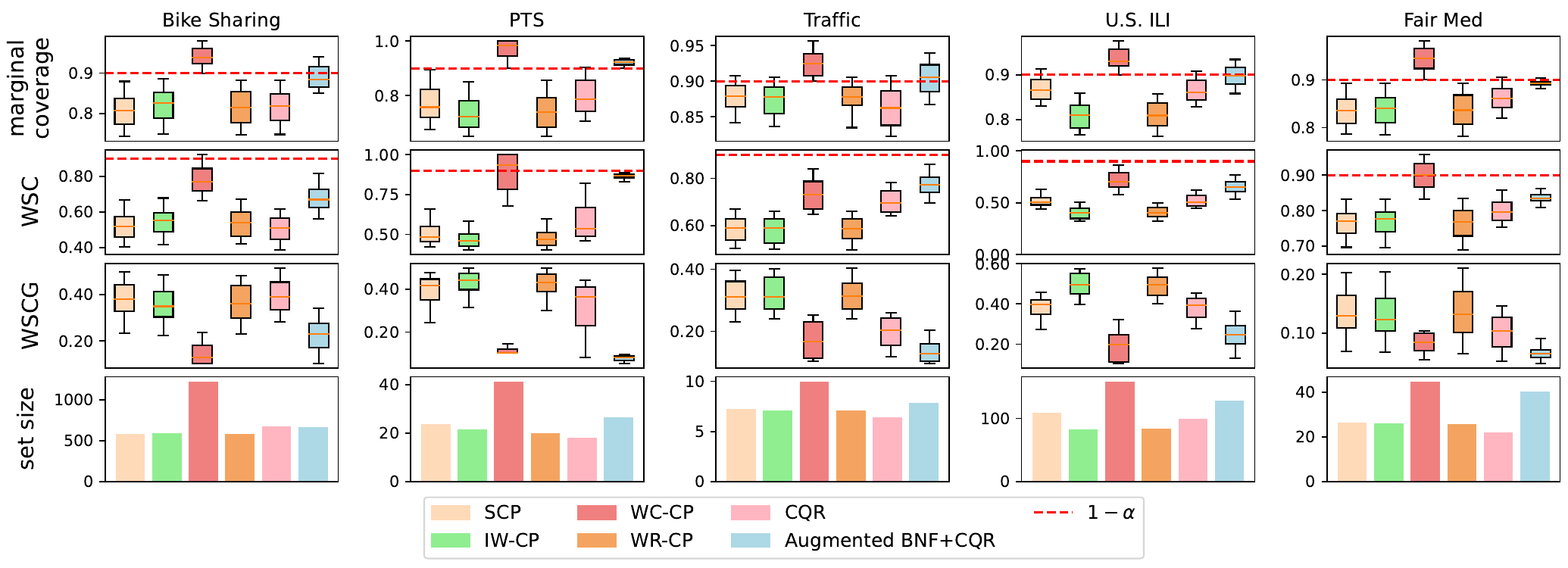}
  \vspace{-5pt}
  \caption{
  Distribution shift is induced by label perturbation. Compared with baselines, our method consistently maintains comparably robust marginal and conditional coverage without a high cost in prediction efficiency. }
  \label{fig: perturbation} 
\end{figure}

\section{Limitations}\label{appendix: limitations}
\subsection{Root-Finding challenges}
Unlike the original BNF, Augmented BNF does not employ a univariate monotonic transformation for the $Y$-branch. As a result, we cannot directly apply Eq.~(\ref{eq: BNF prediction set}) to construct the prediction set. Instead, Augmented BNF relies on Eq.~(\ref{eq: aug BNF prediction set}) to generate $\mathcal{C}_\text{BNF}^\text{aug}(X_{n+1})$, which frames the construction as a root-finding problem. Specifically, let the interval endpoints of the calibrated set $\mathcal{C}_\text{C}(\widebar{X}_{n+1})$ be denoted by $y_\text{lo}$ and $y_\text{hi}$. Then, Eq.~(\ref{eq: aug BNF prediction set}) requires solving a root-finding problem to identify the pre-images of these endpoints under the learned transformation. In particular, we need to find the values of $y$ that satisfy the following equations
\begin{equation}\label{eq: root finding}
    f^{\text{aug}}_{\theta_Y}(y;\varepsilon_{n+1})=y_\text{lo};\quad f^{\text{aug}}_{\theta_Y}(y;\varepsilon_{n+1})=y_\text{hi}.
\end{equation}
In practice, we construct a prediction set by evaluating a coarse grid of 1000 candidate $y$values and checking, via the inverse transformation, whether the corresponding $\bar{y}$ falls inside $\mathcal{C}_\text{C}(\widebar{X}_{n+1})$. Values that satisfy this condition are included in $\mathcal{C}_\text{BNF}^\text{aug}(X_{n+1})$, with the endpoints defining the prediction set boundary. In practice, 1000 candidates are sufficient to accurately capture the boundaries, and adding more candidates does little change to the set size. The procedure is computationally efficient, adding only 0.03 seconds per test sample compared to standard CQR on an RTX 3090.
\subsection{Stochastic prediction sets}
One practical drawback of Augmented BNF lies in its reliance on stochastic augmentation through a random noise $\varepsilon \sim \mathcal{N}(0,1)$, which is used to modulate the $Y$-branch of the Augmented BNF in Eq.~(\ref{eq: normalization via augmented BNF}). While this augmentation introduces flexibility, it also introduces randomness into the transformation. As a result, in Eq.~(\ref{eq: aug BNF prediction set}), the prediction set produced by Augmented BNF for the same input $x$ is no longer deterministic. The set varies across different forward passes depending on the realization of $\varepsilon$. This stochasticity undermines one of the appealing properties of standard conformal prediction: the deterministic and repeatable nature of the prediction set given a test point. In high-stakes domains, such randomness can lead to interpretability challenges or instability in downstream decisions. 

A plausible solution is to perform multiple runs to approximate a stable prediction set, or to fix the random seed so that each test input yields a consistent prediction set across repeated queries.

\subsection{Architectural incompatibility with one-dimensional features}
Another limitation of Augmented BNF stems from its architectural dependency on Real NVP~\citep{dinh2016density}, a type of normalizing flow that is inherently designed for multi-dimensional transformations. Real NVP operates by alternating between dimensions of the input to apply affine coupling layers, as plotted in Figure~\ref{fig: architecture}. This necessitates a feature space $\mathcal{X}$ of at least two dimensions. Consequently, Augmented BNF inherits this constraint: its architecture presumes that the input feature $x$ is multivariate.

In the case where $x$ is one-dimensional, the affine coupling mechanism of Real NVP becomes undefined, rendering the model non-functional. As a result, Augmented BNF cannot be applied to tasks with univariate inputs. This presents a clear barrier for applying to domains, where no natural multivariate feature exists.
One might consider artificially expanding $x$ with noise or engineered features to satisfy the dimensionality requirement, just like the augmented $Y$-branch in Eq.~(\ref{eq: normalization via augmented BNF}).

\subsection{Tuning Bias}
In Section~\ref{sec: augmented BNF}, the calibration set $\mathcal{S}_P$ participates in the training of the Augmented BNF, as described in Algorithm~\ref{alg: Augmented BNF under MSDG}. However, this practice may undermine the rigor of conformal prediction, where calibration data is ideally held out from any training procedure. Despite this concern, similar strategies have been adopted in prior work~\citep{angelopoulos2020uncertainty,dabah2025temperaturescalingconformalprediction,xi2024doesconfidencecalibrationimprove,yang2024selectionaggregationconformalprediction}, often to simplify implementation. Notably,~\cite{zeng2025parametricscalinglawtuning} identifies a parametric scaling law of tuning bias, showing that reusing calibration data introduces a bias that grows with model complexity and diminishes as the calibration set size increases.

To adhere more closely to the theoretical foundations of conformal prediction, a more principled approach would involve randomly partitioning $\mathcal{S}_P$ into two disjoint subsets: one used for training the Augmented BNF and another reserved exclusively for inference. Given the architectural complexity and parameterization of the Augmented BNF, as detailed in Appendix~\ref{appendix: architecture}, such a split is particularly recommended to mitigate the risk of overfitting and maintain robust uncertainty guarantees. 

We conduct an ablation study with a strict, unbiased training–calibration–test split across all datasets. As shown in Figure~\ref{fig: strict split}, WSCG remains largely unchanged, with no noticeable degradation, suggesting that the practical impact of such bias in Section~\ref{sec: application to MSDG} is minimal.

\begin{figure}[h]
\centering
\captionsetup{singlelinecheck = false, skip=5pt, justification=justified}
  \includegraphics[width=1.0\textwidth]{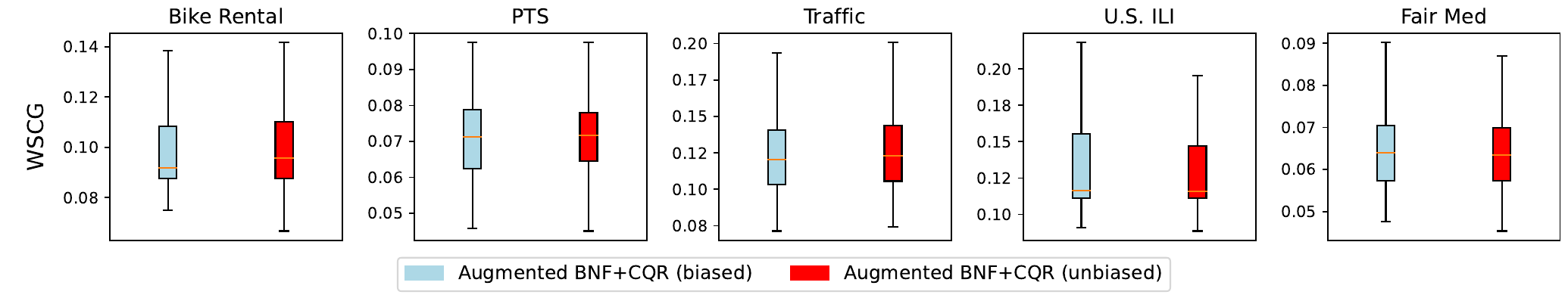}
  \vspace{-5pt}
  \caption{
  Ablation study on the unbiased setting where a strict data split is conducted. The WSCG metric shows no significant degradation.}
  \label{fig: strict split} 
\end{figure}
%%%%%%%%%%%%%%%%%%%%%%%%%%%%%%%%%%%%%%%%%%%%%%%%%%%%%%%%%%%%

\end{document}